\documentclass[letterpaper]{article} 
\usepackage{aaai2026}  
\usepackage{times}  
\usepackage{helvet}  
\usepackage{courier}  
\usepackage[hyphens]{url}  
\usepackage{graphicx} 
\usepackage{natbib}  
\usepackage{caption} 
\usepackage[ruled,vlined,linesnumbered]{algorithm2e}
\usepackage{float}
\usepackage{newfloat}
\usepackage{listings}
\DeclareCaptionStyle{ruled}{labelfont=normalfont,labelsep=colon,strut=off} 
\floatstyle{ruled}
\newfloat{listing}{tb}{lst}{}
\floatname{listing}{Listing}
\usepackage{amsmath}
\usepackage{amsthm}  
\newtheorem{theorem}{Theorem}
\usepackage{graphicx}  
\usepackage{tcolorbox} 
\usepackage{multirow}
\usepackage{colortbl}
\usepackage{bbding}
\usepackage{makecell}
\usepackage{array}  

\usepackage{amssymb,booktabs}
\usepackage{tabularx}
\usepackage{bbding}
\usepackage{pifont}
\usepackage{makecell}

\usepackage{enumitem}

\title{Gradient Mirage: Trainable yet Label-Unidentifiable Gradients in Large Language Model Split Learning}

\author {
    Shiyu Miao\textsuperscript{\rm 1},
    Yunlong Mao\textsuperscript{\rm 1}\thanks{Corresponding Author},
    Zirui Huang\textsuperscript{\rm 1},
    Liang Yao\textsuperscript{\rm 2}, 
    Tianshuo Zheng\textsuperscript{\rm 3}, \\
    Yanhui Gu\textsuperscript{\rm 4},
    Fan Liu\textsuperscript{\rm 2},
    Sheng Zhong\textsuperscript{\rm 1},
}
\affiliations {
    \textsuperscript{\rm 1}State Key Laboratory for Novel Software Technology, Nanjing University \\
    \textsuperscript{\rm 2}College of Computer Science and Software Engineering, Hohai University \\
    \textsuperscript{\rm 3}School of Mathematics, Nanjing University
    \textsuperscript{\rm 4} Nanjing Normal University\\
    maoyl@nju.edu.cn, zhongsheng@nju.edu.cn
}

\usepackage{bibentry}

\usepackage{bm}
\newtheorem{definition}{Definition}
\begin{document}

\maketitle

\begin{abstract}
Gradient matching attacks (GMAs) in LLM split learning (SL) rely on a critical yet underexplored assumption: the gradient exposed at the split interface is a faithful derivative of the client’s full-label training objective. This gradient–objective consistency allows a curious server to recover private labels by searching for a sequence whose induced gradient explains the observation. We propose Gradient Mirage, a defense that breaks this consistency without discarding the optimization utility of the backward signal. Our key idea is to induce the adversary to solve a misspecified inverse problem, in which no plausible label sequence in the sequence space can explain the observed gradients.
Concretely, Gradient Mirage achieves this by inducing inconsistency across three dimensions: objective, direction, and scale. 
Selective Autoregressive Supervision derives the exposed gradient from a masked surrogate loss rather than the full-label objective assumed by the attacker; Scale Blinding then applies randomized multiplicative rescaling, obscuring the gradient's natural magnitude; and Directional Privatization further randomizes the gradient direction while preserving its magnitude through the von Mises–Fisher (vMF) mechanism under a directional metric differential privacy guarantee.
Crucially, utility is preserved: the Top segment still learns from all target tokens via Dual-Track Backpropagation, the exposed gradient remains informative since each supervised token retains its complete autoregressive context, and Bottom-Gradient Recovery restores the effective gradient for Bottom-segment optimization.
Extensive experiments show that Gradient Mirage provides substantially stronger protection than existing defenses under comparable fine-tuning performance, achieving a better privacy-utility trade-off.

\end{abstract}

\begin{links}
    \link{Code}{https://github.com/StevenMsy/GMA-SL}
\end{links}

\begin{figure}[t]
    \centering
    \includegraphics[width=0.95\linewidth]{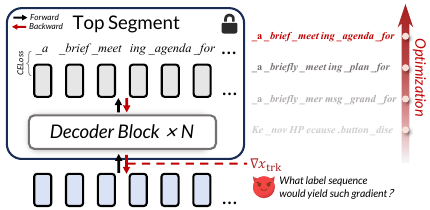}
    \caption{\textit{Illustration of label leakage in label-shielded split learning.} Given the gradient returned from the client’s Top segment to the Trunk segment, a curious-but-honest server can recover the label sequence by optimization.}
    \label{fig:fig1}
\end{figure}

\section{Introduction}
Machine Learning-as-a-Service (MLaaS) has become a dominant paradigm for deploying large-scale models~\cite{mlaas1, mlaas2}, allowing data owners to outsource computation and reduce infrastructure overhead. In its training setting, commonly known as Training-as-a-Service (TaaS), Split Learning (SL)~\cite{SL2, SL3, SL1, SplitLLM, SL4, SL5} mitigates direct data exposure by partitioning model training between the client and server. We focus on LLM supervised fine-tuning under \emph{Label-Shielded Split Learning}~\cite{unleashing}, where the compute-intensive Trunk segment is outsourced to the server, while the Bottom and Top segments remain on the client to keep private inputs, labels, and supervised loss local.

However, prior research~\cite{SplitLLM} shows that this LLM-SL system fails to provide effective privacy protection against a \textit{curious-but-honest} server. In particular, the client’s private data becomes highly vulnerable to GMAs~\cite{dlg, idlg, tag, lamp, film} at the split interface between the Trunk and Top segment. As illustrated in Fig.~\ref{fig:fig1}, the adversary’s guiding question is essentially: \textit{"What label sequence would yield such gradients?"} By optimizing for a label sequence whose interface gradients match the observed ones, it can recover sensitive training content even without explicit access to labels.

We term this attack, launched at the split interface between the Trunk and Top segment, GMA-SL. GMA-SL constitutes a particularly compelling variant of classical GMA: We delineate its relationship to, and key distinctions from classical GMA in Section~\ref{sec2.2:classical GMA}, and provide the concrete attack instantiation and implementation details in Section~\ref{sec2.3:GMA-SL}. 

GMA-SL poses a particularly severe gradient leakage threat in LLM-SL. Unlike classical GMA, it matches activation gradients at the split interface, which retain explicit batch structure and are closely tied to local supervision. Moreover, because the cut-layer representation is already available, the adversary primarily optimizes the label sequence rather than jointly reconstructing inputs and labels, making the attack both simpler and harder to defend against.  

Motivated by this challenge, we propose Gradient Mirage, a defense
method tailored to GMA-SL. Its key insight is that \emph{what the model learns} and \emph{what the exposed gradient reveals} need not coincide. Gradient Mirage deliberately decouples
these two objectives. Under the constraint of minimally affecting
optimization, we seek an exposed signal residing in a ``gray zone''
of the gradient space: to the server-side optimizer, it remains a
serviceable backward signal for trunk updates; to the adversary, it
looks like a gradient, yet no suitable label sequence in the
sequence space can explain it under the attacker's assumed objective.
Consequently, an adversary performing gradient matching against this
signal is in effect solving a misspecified inverse problem---one whose
solution set is not anchored to the true labels.

Gradient Mirage realizes this design by inducing inconsistency across three dimensions. Dual-Track Backpropagation
confines the full-label loss $\mathcal{L}_{\mathrm{F}}$ to local
learning at the Top segment, while the transmitted gradient is
generated by a masked surrogate loss $\mathcal{L}_{\mathrm{M}}$
via Selective Autoregressive Supervision, yielding \emph{Objective
Inconsistency} ($\mathcal{L}_{\mathrm{F}} \neq
\mathcal{L}_{\mathrm{M}}$): making the attacker's assumed objective misspecified at its root. Directional Privatization then randomizes
the gradient direction via vMF mechanism~\cite{DirDP,vonMises,fisher} with a directional metric differential privacy guarantee, inducing
\emph{Directional Inconsistency}
($\frac{g_{\mathrm{release}}}{\|g_{\mathrm{release}}\|} \neq
\frac{\nabla_x \mathcal{L}_{\mathrm{M}}}{\|\nabla_x
\mathcal{L}_{\mathrm{M}}\|}$). Scale Blinding applies randomized
multiplicative rescaling, introducing \emph{Scale Inconsistency}
($\|g_{\mathrm{release}}\| \neq \|\nabla_x
\mathcal{L}_{\mathrm{M}}\|$): even when the
direction is approximately matched, the gradient's natural magnitude
is no longer trustworthy. Meanwhile, optimization compatibility is supported by design: the Top segment still learns under full supervision, the vMF mechanism retains the gradient magnitude with a provable lower bound on the amplitude signal-to-noise ratio, and Bottom-Gradient Recovery restores the effective gradient for bottom optimization.
Gradient Mirage thereby disrupts the label-identifiability of the
exposed gradient while retaining stable and effective SL.

We present, to the best of our knowledge, the first defense shown to be effective against GMA-SL, a particularly severe gradient leakage threat in LLM-SL. As GMA-SL constitutes a specialized form of GMA in the SL setting, we adapt representative GMA defenses and a stringent sequence-level differential privacy (DP) mechanism as practical baselines, demonstrating that our method consistently achieves superior privacy-utility trade-offs.

The contributions are summarized as follows:
$\boldsymbol{(i)}$ We identify gradient--objective consistency as the key vulnerability underlying GMA-SL and provide a systematic characterization of this attack in label-shielded LLM-SL, clarifying its attack surface and key distinctions from classical GMA. $\boldsymbol{(ii)}$ We propose Gradient Mirage, a defense that decouples what the model learns from what the gradient reveals, injecting objective, directional, and scale inconsistencies into the released gradient while preserving optimization utility.  $\boldsymbol{(iii)}$ We conduct extensive experiments across multiple LLMs and datasets, adapting representative GMA defenses and a rigorous DP method as baselines. Gradient Mirage consistently achieves a superior privacy-utility trade-off.

\section{Preliminary}
\label{sec:related_work}
\begin{figure}[t]
    \centering
    \includegraphics[width=0.88\linewidth]{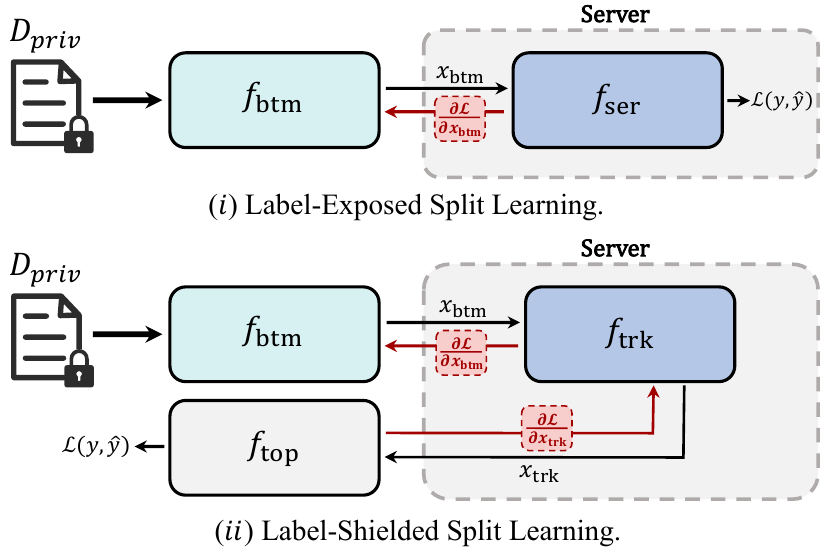}
    \caption{\textit{Two classical split learning paradigms:} \textbf{label-exposed split learning} and \textbf{label-shielded split learning.}}
    \label{fig:SL-Arch}
\end{figure}

\subsection{Split Learning}

SL partitions a neural network across clients and a central server, so that clients hold only a small portion of the model while the majority of parameters—and thus the computational burden—reside on the server. As illustrated in Fig.~\ref{fig:SL-Arch}, typical SL can be categorized into two paradigms: (I) \textbf{Label-Exposed SL,} where the network is split into two parts and clients disclose labels to the server; (II) \textbf{Label-Shielded SL,} where labels remain on-device and the network is partitioned into three segments, with the front and tail placed on the client and the middle trained on the server.

In LLM supervised fine-tuning, next-token labels are obtained by shifting the input sequence by one position, so exposing labels would effectively reveal the private text. We therefore consider \textbf{Label-Shielded SL}.
Given a client's private dataset \(D_{\mathrm{pri}}\), the overall network \(f_{\theta}\) is partitioned into a \emph{Bottom} \(f_{\mathrm{btm}}\), a \emph{Trunk} \(f_{\mathrm{trk}}\), and a \emph{Top} \(f_{\mathrm{top}}\). The server hosts \(f_{\mathrm{trk}}\), while the client retains \(f_{\mathrm{btm}}\) and \(f_{\mathrm{top}}\). 
For a mini-batch \(x \in D_{\mathrm{pri}}\), the client transmits \(x_{\mathrm{btm}} = f_{\mathrm{btm}}(x)\) to the server, and receives
\(x_{\mathrm{trk}} = f_{\mathrm{trk}}(x_{\mathrm{btm}})\) from the server.
The client then obtains the prediction
\(\hat{y} = f_{\mathrm{top}}(x_{\mathrm{trk}})\) and computes the training objective \(\mathcal{L}(\hat{y}, y)\) with label $y$ kept local.
During backpropagation, the exchanged gradients
\(\partial \mathcal{L} / \partial x_{\mathrm{trk}}\) allow the server to update the Trunk, while the client updates the Bottom and Top using \(\partial \mathcal{L} / \partial x_{\mathrm{btm}}\) and \(\partial \mathcal{L} / \partial x_{\mathrm{top}}\).

\begin{figure}[t]
    \centering
    \includegraphics[width=0.82\linewidth]{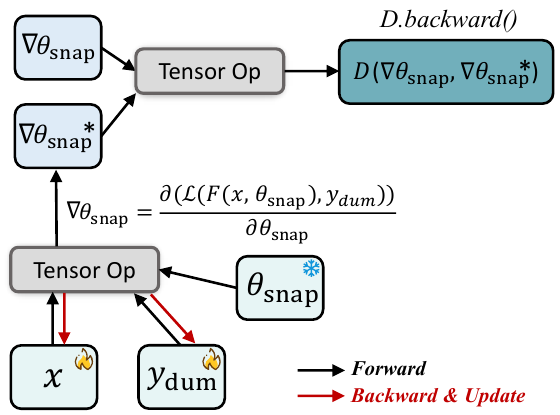}
    \caption{\textit{Classical gradient-matching attack with a frozen parameter snapshot.} At training step~$t$, the adversary fixes the model parameters $\boldsymbol{\theta}^{(t)}$ (the snapshot) and optimizes dummy inputs \textbf{$x_{dum}$} and labels \textbf{$y_{dum}$}.}
    \label{fig:tra_gma}
\end{figure}

\begin{figure*}[t]
    \centering
    \includegraphics[width=0.98\linewidth]{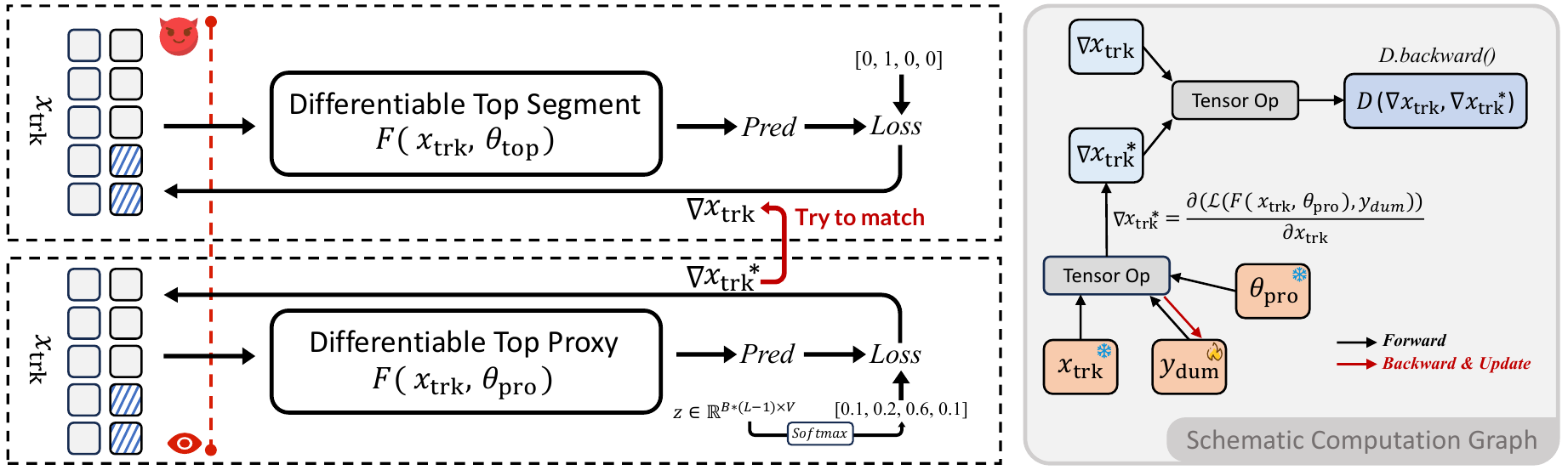}
    \caption{\textit{Gradient matching attack in split learning (GMA-SL).} The adversary observes the gradient with respect to the cut-layer representation $\nabla_{}x_{trk}$, produced by the victim’s Top segment, then optimizes a dummy label $y_{dum}$ using a differentiable proxy top model $F(\cdot;\boldsymbol{\theta}_{\mathrm{pro}})$, such that the proxy-induced gradient $\nabla_{x_{\mathrm{trk}}}\mathcal{L}\!\big(F(x_{\mathrm{trk}};\boldsymbol{\theta}_{\mathrm{pro}}),\tilde{y}\big)$ matches the observed one. The schematic computation graph is shown on the right side.}
    \label{overview}
\end{figure*}

\subsection{Classical Gradient Matching Attack}
\label{sec2.2:classical GMA}
Classical GMA arises in distributed training, where multiple participants collaboratively optimize a shared model while keeping their data local. During training, they exchange model updates—typically gradients—rather than raw data. This gradient exchange creates the primary attack surface exploited by gradient-matching methods. As illustrated in Fig.~\ref{fig:tra_gma}, classical GMA assumes access to the model snapshot $\theta^{(t)}$ used to generate the observed gradient; reconstruction is then performed with respect to this fixed snapshot.

Let $F(\cdot;\theta)$ be the model and $\mathcal{L}(\cdot,\cdot)$ the training loss. At step $t$, the shared gradient is
\begin{equation}
\mathbf{g}^{(t)} \;=\; \nabla_{\boldsymbol{\theta}} \, \mathcal{L}\!\left( F\!\left(x^{(t)};\boldsymbol{\theta}^{(t)}\right), y^{(t)} \right),
\end{equation}
where $\boldsymbol{\theta}^{(t)}$ denotes the model-parameter snapshot at step $t$, and
$\left(x^{(t)}, y^{(t)}\right)$ is the private minibatch
that generated $\mathbf{g}^{(t)}$.
The adversary introduces dummy variables $(x_{dum}, y_{dum})$ and solves:

\begin{equation}
(x^{*}, y^{*})
\;=\;
\arg\min_{\tilde{x},\,\tilde{y}}\;
D\!\left(
\nabla_{\boldsymbol{\theta}}\,\mathcal{L}\!\left(F(\tilde{x};\boldsymbol{\theta}^{(t)}),\tilde{y}\right),
\mathbf{g}^{(t)}
\right),
\end{equation}
where $D(\cdot,\cdot)$ is a gradient-discrepancy measure (e.g., $D(\mathbf{a},\mathbf{b})=\lVert \mathbf{a}-\mathbf{b}\rVert_2^2$, cosine distance, or a weighted combination).
The optimization is carried out by gradient descent on the dummy variables:
\begin{equation}
    \begin{aligned}
    (\tilde{x}, \tilde{y}) \leftarrow (\tilde{x}, \tilde{y})
    -
    \eta \nabla_{(\tilde{x}, \tilde{y})}
    D\!\left(
    \nabla_{\boldsymbol{\theta}} \mathcal{L}\!\left(F(\tilde{x};\boldsymbol{\theta}^{(t)}), \tilde{y}\right),
    \mathbf{g}^{(t)}
    \right),
    \end{aligned}
\end{equation}
while keeping $\boldsymbol{\theta}^{(t)}$ unchanged throughout the optimization.

\subsection{Gradient Matching Attack in SL}
\label{sec2.3:GMA-SL}
The GMA in the SL setting can be viewed as a variant of the classical GMA. Unlike classical GMA, matching is performed in the \emph{activation space} at the split interface rather than directly on $\nabla_{\boldsymbol{\theta}}$. The schematic overview of the entire pipeline is shown in Fig.~\ref{overview}.

Let the Top segment be $F(\cdot;\theta_{top})$ and the top proxy held by the adversary $F(\cdot;\theta_{pro})$. Given label $y$, the client locally computes prediction $\hat{y}=F(x_{trk};\theta_{top})$ and loss $\mathcal{L}(\hat{\mathbf{y}},\mathbf{y})$. The gradient propagated back to the cut layer is given by:
\begin{equation}
\mathbf{g}_{\mathrm{trk}}\;=\;
\nabla_{x_{\mathrm{trk}}}\,
\mathcal{L}\!\left(F(x_{\mathrm{trk}};\boldsymbol{\theta}_{\mathrm{top}}),\,y\right).
\end{equation}

The attacker is assumed to observe $\mathbf{g}_{\mathrm{trk}}$. The goal is to recover a plausible label $\tilde{y}$ that yields a matching gradient signal. The attacker uses a differentiable proxy Top segment $F(\cdot;\theta_{\mathrm{pro}})$ and introduces dummy
variables $\tilde{\mathbf{x}}_{\mathrm{trk}}$ and $\tilde{\mathbf{y}}$. The proxy-induced gradient is
\begin{equation}
\tilde{\mathbf{g}}_{\mathrm{trk}}(\tilde{y})
=
\nabla_{x_{\mathrm{trk}}}\,
\mathcal{L}\!\left(
F(x_{\mathrm{trk}};\theta_{\mathrm{pro}}),
\tilde{y}
\right),
\end{equation}
The attack solves the following matching problem:
\begin{equation}
\tilde{y}^{*}
=
\arg\min_{\tilde{y}}
D\!\left(
\mathbf{g}_{\mathrm{trk}},
\tilde{\mathbf{g}}_{\mathrm{trk}}(x_{\mathrm{trk}},\tilde{y})
\right),
\end{equation}
where $D(\cdot,\cdot)$ is a discrepancy measure.

The optimization is carried out by gradient descent on the
dummy variables:
\begin{equation}
\tilde{y}
\leftarrow\tilde{y}-\eta\,
\nabla_{\tilde{y}}
D\!\left(
\mathbf{g}_{\mathrm{trk}},
\tilde{\mathbf{g}}_{\mathrm{trk}}(x_{\mathrm{trk}},\tilde{y})
\right),
\end{equation}
while $\theta_{\mathrm{pro}}$ is held fixed during the inner-loop optimization.

Compared with classical GMA, GMA-SL yields the following advantages. $\boldsymbol{(i)}$ \textbf{Batch-separable matching.} In SL, the observed gradient
$\nabla_{x_{\mathrm{trk}}}\mathcal{L}\in\mathbb{R}^{B\times L\times D}$
retains explicit batch structure, allowing samples to be matched
independently except for the scalar coupling introduced by batch-averaged
loss. The only coupling arises from the scalar coefficient induced by averaging the loss over the batch (discussed in Section~2.4). This behavior contrasts with classical GMA, where the attacker matches gradients with respect to model parameters (e.g., $\nabla_{\mathbf{W}}$), which implicitly aggregates information across the batch and can entangle samples. $\boldsymbol{(ii)}$ \textbf{A simpler optimization objective.} With $x_{\mathrm{trk}}$ fixed, the attacker mainly needs to optimize a dummy label $\tilde{y}$. Compared to classical GMA-where both the input $\tilde{x}$ and label $\tilde{y}$ are typically optimized-this reduces the number of free variables and mitigates optimization interference. At the same time, SL-based GMA also introduces new challenges. $\bm{(i)}$ \textbf{Limited gradient observability. } The attacker can only match gradients defined at the split interface (i.e., gradients with respect to activations), rather than the full gradient over all model parameters. In classical GMA, gradients from deeper layers can sometimes be weighted or exploited to improve reconstruction, whereas GMA-SL is restricted to $\nabla_{\mathbf{x_{\mathrm{trk}}}}$. $\bm{(ii)}$ \textbf{Semi-white-box scenario limitation.} The attacker does not have access to an exact snapshot of the victim’s top model parameters, making the setting non–white-box. Instead, we consider the semi–white-box assumption~\cite{SplitLLM} in which the attacker only knows the initialization (pretrained weight) used for LLM fine-tuning.

\subsection{Loss in Autoregressive Decoder Structure}
In this section, we formalize the autoregressive training loss used throughout the paper and clarify how the loss is computed by the client and the adversary under our threat model introduced in section~\ref{sec3:threat_model}.

For a mini-batch of size \(B\) and sequence length \(L\), let
$\mathbf{H}\in\mathbb{R}^{B\times L\times D}$
denote the input feature tensor (token-wise hidden states) at the input of the language-modeling head, and let $\mathbf{M}\in\{0,1\}^{B\times L}$
be the attention mask, where \(\mathbf{M}_{b,\ell}=1\) indicates that position \(\ell\) in example \(b\) is a valid token contributing to training, and \(\mathbf{M}_{b,\ell}=0\) indicates padding or an ignored position.
The model produces next-token logits
\begin{equation}
    \mathbf{L} = f_{\mathrm{LM}}(\mathbf{H};\boldsymbol{\theta}) \in \mathbb{R}^{B\times L\times V},
\end{equation}
where \(V\) is the vocabulary size and \(\boldsymbol{\theta}\) are model parameters.

Let \(\mathbf{Y}\in\{1,\dots,V\}^{B\times L}\) be the token id matrix for the same sequences.
With the standard one-step shift, the training target for position \(\ell\) is \(\mathbf{Y}_{b,\ell+1}\).
Define the per-token cross-entropy
\begin{equation}
    \ell_{b,\ell}(\boldsymbol{\theta})=-\log \mathrm{Softmax}\!\left(\mathbf{L}_{b,\ell,:}\right)_{\mathbf{Y}_{b,\ell+1}}.
\end{equation}
We mask out invalid positions and (optionally) the final position, yielding the masked autoregressive loss
\begin{equation}
\label{equ:loss_sum}
    \mathcal{L}_{\mathrm{AR}}(\boldsymbol{\theta})=
    \frac{1}{\sum_{b=1}^{B}\sum_{\ell=1}^{L-1}\mathbf{M}_{b,\ell+1}}
    \sum_{b=1}^{B}\sum_{\ell=1}^{L-1}
    \mathbf{M}_{b,\ell+1}\,\ell_{b,\ell}(\boldsymbol{\theta}).
\end{equation}

For the clients, they perform the computation in Eq.~\ref{equ:loss_sum} locally at the end of the Top segment. For the adversary, they execute:
\begin{equation}
    \begin{aligned}
    &\tilde{\mathbf{Z}} = F(x_{\mathrm{trk}};\theta_{\mathrm{pro}}) \in \mathbb{R}^{B\times L\times V}, \\
    &\tilde{\ell_{b,\ell}}(\boldsymbol{\theta})=-\log \mathrm{Softmax}\!\left(\tilde{\mathbf{L}}_{b,\ell,:}\right)_{\tilde{\mathbf{Y}}_{b,\ell+1}}, \\
    \mathcal{L}_{\mathrm{AR}}(\boldsymbol{\theta})=&
    \frac{1}{\sum_{b=1}^{B}\sum_{\ell=1}^{L-1}\mathbf{M}_{b,\ell+1}}
    \sum_{b=1}^{B}\sum_{\ell=1}^{L-1}
    \mathbf{M}_{b,\ell+1}\,\tilde{\ell_{b,\ell}}(\boldsymbol{\theta}).
    \end{aligned}
\end{equation}
The adversary fully knows the normalization factor
\(N=\sum_{b=1}^{B}\sum_{\ell=1}^{L-1}\mathbf{M}_{b,\ell+1}\).
This is because, in addition to the cut-layer representation \(\mathbf{x}_{\mathrm{btm}}\),
the attention mask \(\mathbf{M}\) must also be transmitted to the server: the server-side
trunk requires \(\mathbf{M}\) to construct attention score matrices and to carry out the
forward pass. Since this information is a functional prerequisite for executing the
protocol, even the weakest adversarial assumptions in the LLM-SL setting typically grant
the server access to \(\mathbf{M}\), and hence to \(N\).

\subsection{Batch-Separable Matching}
In this section, we elaborate on advantage $\boldsymbol{(i)}$ of the adversary in GMA-SL, as introduced in Section~\ref{sec2.3:GMA-SL}, namely why samples within the same mini-batch do not interfere with one another during gradient matching.
We begin with the computational details of gradient matching in SL.

Let \(B\) be the batch size, \(L\) the sequence length, and \(V\) the vocabulary size. The server observes $\boldsymbol{(i)}$ the smashed data \(x_{\mathrm{trk}}\in\mathbb{R}^{B\times L\times D}\), $\boldsymbol{(ii)}$ the attention mask \(\mathbf{M}\in\{0,1\}^{B\times L}\), and $\boldsymbol{(iii)}$ the back-propagated gradient w.r.t.\ the smashed data \(\mathbf{G} = \mathbf{g}_{\mathrm{trk}}\in\mathbb{R}^{B\times L\times D}\), where
\(
\mathbf{g}_{\mathrm{trk}} \;=\; \nabla_{x_\mathbf{trk}}\,\mathcal{L}_{\mathrm{AR}}(x_\mathbf{trk};\theta_{\mathrm{top}}) .\) The attacker optimizes a dummy soft label tensor \(\tilde{\mathbf{Z}} \in \mathbb{R}^{B\times L\times V}\) (logits in label space), which is transformed into a distribution via temperature-scaled softmax:
\begin{equation}
\mathbf{P} \;=\; \mathrm{Softmax}\!\left(\tilde{\mathbf{Z}}/\tau\right)\in[0,1]^{B\times L\times V},
\end{equation}
where \(\tau>0\) is a temperature hyperparameter. For the \textit{next-token soft target distribution}, the per-token cross-entropy is
\begin{equation}
\ell_{b,\ell}
\;=\;
-\sum_{v=1}^{V} \mathbf{P}_{b,\ell+1,v}\,
\log \mathrm{Softmax}\!\left(\tilde{\mathbf{L}}_{b,\ell,:}\right)_{v}.
\end{equation}

Masking is applied to ignore padding positions (and effectively excludes the final token by using the \(\ell+1\) target):
\begin{equation}
\mathcal{L}_{\mathrm{soft}}(x_{\mathrm{trk}},\tilde{\mathbf{Z}};\boldsymbol{\theta}_{\mathrm{pro}})
\;=\;
\frac{1}{N}
\sum_{b=1}^{B}\sum_{\ell=1}^{L-1}
\mathbf{M}_{b,\ell+1}\,\ell_{b,\ell},
\end{equation}
where $N=\sum_{b=1}^{B}\sum_{\ell=1}^{L-1}\mathbf{M}_{b,\ell+1}$.
From $\mathcal{L}_{\mathrm{soft}}$, the attacker computes the gradient with respect to the observed smashed data $x_{\mathrm{trk}}$:
\begin{equation}
\tilde{\mathbf{g}}_{\mathrm{trk}}(\tilde{\mathbf{Z}})
=\nabla_{x_\mathbf{trk}}\,\mathcal{L}_{\mathrm{soft}}(x_\mathrm{trk}, \tilde{\mathbf{Z}}; {\theta}_{\mathrm{pro}}) \in \mathbb{R}^{B\times L\times D}.
\end{equation}

The attack then minimizes a discrepancy between the induced gradient \(\tilde{\mathbf{G}}(\tilde{\mathbf{Z}})\) and the observed gradient \(\mathbf{G}\), \textbf{where $\widetilde{\mathbf{G}}(\widetilde{\mathbf{Z}}) \triangleq \tilde{\mathbf{g}}_{\mathrm{trk}}(\widetilde{\mathbf{Z}})$ is exactly the gradient defined in Eq. (14); we adopt the uppercase notation hereafter to emphasize its role as a tensor being matched against $\mathbf{G}$.}
Taking TAG as an illustrative example, we define the gradient-matching objective as
\begin{equation}
\begin{aligned}
\hspace{0.5em}\mathcal{L}_{\mathrm{TAG}}(\tilde{\mathbf{Z}})
=&
\sum_{b=1}^{B}\sum_{\ell=1}^{L}\mathbf{M}_{b,\ell}
\Big(
\beta \big\|\tilde{\mathbf{G}}_{b,\ell,:}(\tilde{\mathbf{Z}})-\mathbf{G}_{b,\ell,:}\big\|_{2}^{2}
\\&+ (1-\beta)\big\|\tilde{\mathbf{G}}_{b,\ell,:}(\tilde{\mathbf{Z}})-\mathbf{G}_{b,\ell,:}\big\|_{1}
\Big),
\end{aligned}
\end{equation}
where \(\beta\in[0,1]\) controls the trade-off between the \(\ell_{2}\) and \(\ell_{1}\) discrepancies.
The key optimization variable is \(\tilde{\mathbf{Z}}\), but the matching loss depends on \(\tilde{\mathbf{Z}}\) only through the intermediate gradient \(\mathbf{G}(\tilde{\mathbf{Z}})=\nabla_{x_{\mathrm{trk}}}\mathcal{L}_{\mathrm{soft}}\). Therefore, updating \(\tilde{\mathbf{Z}}\) requires differentiating \emph{through a gradient}, i.e., computing a mixed second-order derivative:

\begin{equation}
\begin{aligned}
\hspace{1em}\nabla_{\tilde{\mathbf{Z}}}\mathcal{L}_{\mathrm{TAG}}(\tilde{\mathbf{Z}})
&=
\frac{\partial \mathcal{L}_{\mathrm{TAG}}}{\partial \tilde{\mathbf{G}}}
\cdot
\underbrace{\frac{\partial \tilde{\mathbf{G}}}{\partial \tilde{\mathbf{Z}}}}_{\text{requires 2nd-order terms}} \\
=
\frac{\partial \mathcal{L}_{\mathrm{TAG}}}{\partial \tilde{\mathbf{G}}}&
\cdot
\frac{\partial}{\partial \tilde{\mathbf{Z}}}
\Big(
\nabla_{x_\mathrm{trk}}\mathcal{L}_{\mathrm{soft}}(x_\mathrm{trk},\tilde{\mathbf{Z}};\boldsymbol{\theta}_\mathrm{pro})
\Big) \\
=\frac{\partial \mathcal{L}_{\mathrm{TAG}}}{\partial \tilde{\mathbf{G}}}&
\cdot\nabla^2_{\tilde{\mathbf{Z}},\,x_\mathbf{trk}} \mathcal{L}_{\mathrm{soft}}(x_\mathrm{trk}, \tilde{\mathbf{Z}};\boldsymbol{\theta}_\mathrm{pro}).
\end{aligned}
\end{equation}

In the implementation, we obtain \(\mathbf{G}(\tilde{\mathbf{Z}})\) via automatic differentiation while retaining the computational graph (\texttt{create\_graph=True})\footnote{In PyTorch, setting \texttt{create\_graph=True} in \texttt{torch.\allowbreak autograd.\allowbreak grad} constructs a computation graph for the differentiation operation itself, i.e., it treats the returned gradient as a differentiable output of an autograd-tracked operator.}, which enables backpropagation from \(\mathcal{L}_{\mathrm{TAG}}\) to \(\tilde{\mathbf{Z}}\).

Finally, the adversary decodes discrete token predictions by $
\tilde{\mathrm{Y}}_{b,\ell} \;=\; \arg\max_{v\in[V]} \tilde{\mathbf{Z}}_{b,\ell,v}$,
and ignores positions where \(\mathbf{M}_{b,\ell}=0\).

We now return to the main theme of this section: when optimizing the dummy label logits \(\tilde{\mathbf{Z}}\), samples within the same mini-batch are \emph{batch-separable} and hence do not affect each other, except through the global normalization constant \(N\), which is known to the adversary.

The key observation comes directly from the computational graph. For each sample \(b\in[B]\), the server-side proxy Top segment and the shifted soft-target loss induce a per-sample subgraph
\begin{equation}
\begin{aligned}
x_{\mathrm{trk}}^{(b)},\,\mathbf{M}^{(b)} \;\rightarrow\; & \tilde{\mathbf{L}}^{(b)} \;=\; F\!\big(x_{\mathrm{trk}}^{(b)},\,\mathbf{M}^{(b)};\,\theta_{\mathrm{pro}}\big),\\
\tilde{\mathbf{Z}}^{(b)} \; \rightarrow\; &\mathbf{P}^{(b)} \;=\; \mathrm{Softmax}\!\big(\tilde{\mathbf{Z}}^{(b)}/\tau\big),\\
\big(\tilde{\mathbf{L}}^{(b)},\,\mathbf{P}^{(b)},\,\mathbf{M}^{(b)}\big) \;&\rightarrow\; \mathcal{L}_{\mathrm{soft}}^{(b)} \;\rightarrow\; \tilde{\mathbf{g}}_{\mathrm{trk}}^{(b)} \;=\; \nabla_{x_{\mathrm{trk}}^{(b)}} \mathcal{L}_{\mathrm{soft}}.
\end{aligned}
\end{equation}
and these subgraphs are disjoint across different \(b\)'s because modern LLM implementations execute the forward/backward pass independently per example and only stack tensors along the batch dimension for efficiency. Concretely, the masked shifted loss decomposes as
\begin{equation}
\begin{aligned}
\mathcal{L}_{\mathrm{soft}}(x_{\mathrm{trk}},\tilde{\mathbf{Z}};\theta_{\mathrm{pro}})
&=
\frac{1}{N}\sum_{b=1}^{B}\sum_{\ell=1}^{L-1}\mathbf{M}_{b,\ell+1}\,\ell_{b,\ell}\\
&=
\frac{1}{N}\sum_{b=1}^{B}\mathcal{L}_{\mathrm{soft}}^{(b)},
\end{aligned}
\end{equation}
where \(\mathcal{L}_{\mathrm{soft}}^{(b)} \triangleq \sum_{\ell=1}^{L-1}\mathbf{M}_{b,\ell+1}\,\ell_{b,\ell}\) depends only on \((\mathbf{x}_{\mathrm{trk}}^{(b)},\mathbf{M}^{(b)},\tilde{\mathbf{Z}}^{(b)})\).
As a result, the dummy gradient tensor factorizes across the batch:
\begin{equation}
\begin{aligned}
\hspace{1.5em}\tilde{\mathbf{g}}_{\mathrm{trk}}(\tilde{\mathbf{Z}})
&=
\nabla_{x_{\mathrm{trk}}}\mathcal{L}_{\mathrm{soft}}(x_{\mathrm{trk}},\tilde{\mathbf{Z}};\theta_{\mathrm{pro}})\\
=&
\frac{1}{N}\Big[
\nabla_{x_{\mathrm{trk}}^{(1)}}\mathcal{L}_{\mathrm{soft}}^{(1)},
\ldots,
\nabla_{x_{\mathrm{trk}}^{(B)}}\mathcal{L}_{\mathrm{soft}}^{(B)}
\Big],
\end{aligned}
\end{equation}
and there are no cross-sample Jacobian terms. More explicitly, for any \(b\neq b'\),
\[
\frac{\partial\, \tilde{\mathbf{g}}_{\mathrm{trk}}^{(b)}(\tilde{\mathbf{Z}})}{\partial\, \tilde{\mathbf{Z}}^{(b')}} = \mathbf{0},\;
\text{ equivalently }\;
\nabla^2_{\tilde{\mathbf{Z}}^{(b')},\,x_{\mathrm{trk}}^{(b)}}\,\mathcal{L}_{\mathrm{soft}} = \mathbf{0}.
\]
Therefore, the gradient-matching objective (e.g., \(\mathcal{L}_{\mathrm{TAG}}\)) decomposes into a sum of per-sample objectives (up to the shared scalar \(1/N\)):
\begin{equation}
\begin{aligned}
\hspace{2em}\mathcal{L}_{\mathrm{TAG}}(\tilde{\mathbf{Z}})
&=
\sum_{b=1}^{B}\mathcal{L}_{\mathrm{TAG}}^{(b)}(\tilde{\mathbf{Z}}^{(b)}; \mathbf{G}^{(b)})\\
\Rightarrow
\nabla_{\tilde{\mathbf{Z}}^{(b)}}&\mathcal{L}_{\mathrm{TAG}}(\tilde{\mathbf{Z}})
=
\nabla_{\tilde{\mathbf{Z}}^{(b)}}\mathcal{L}_{\mathrm{TAG}}^{(b)}(\tilde{\mathbf{Z}}^{(b)}; \mathbf{G}^{(b)}),
\end{aligned}
\end{equation}
meaning that optimizing \(\tilde{\mathbf{Z}}^{(b)}\) for one sample cannot change the matching loss terms of any other sample \(b'\neq b\).
The only batch-level coupling is the normalization \(N=\sum_{b=1}^{B}\sum_{\ell=1}^{L-1}\mathbf{M}_{b,\ell+1}\), which is fully known to the adversary because \(\mathbf{M}\) must be provided to the server to execute attention in the trunk/top forward pass. Consequently, GMA-SL admits \emph{batch-separable} gradient matching: the attacker may conceptually solve \(B\) independent subproblems (one per sample) within each mini-batch, without cross-sample interference.

\section{Threat Model}
\label{sec3:threat_model}
We consider a curious-but-honest adversary. Specifically, the server is assumed to \emph{faithfully} execute the split learning protocol (e.g., forward/backward computations and parameter updates) without deviating from it. Meanwhile, the server can log any legitimately available information and perform offline analysis or computation to reconstruct private sequences $x \in D_{\mathrm{pri}}$, i.e., a passive inference adversary.

\begin{figure*}[t]
    \centering
    \includegraphics[width=0.98\linewidth]{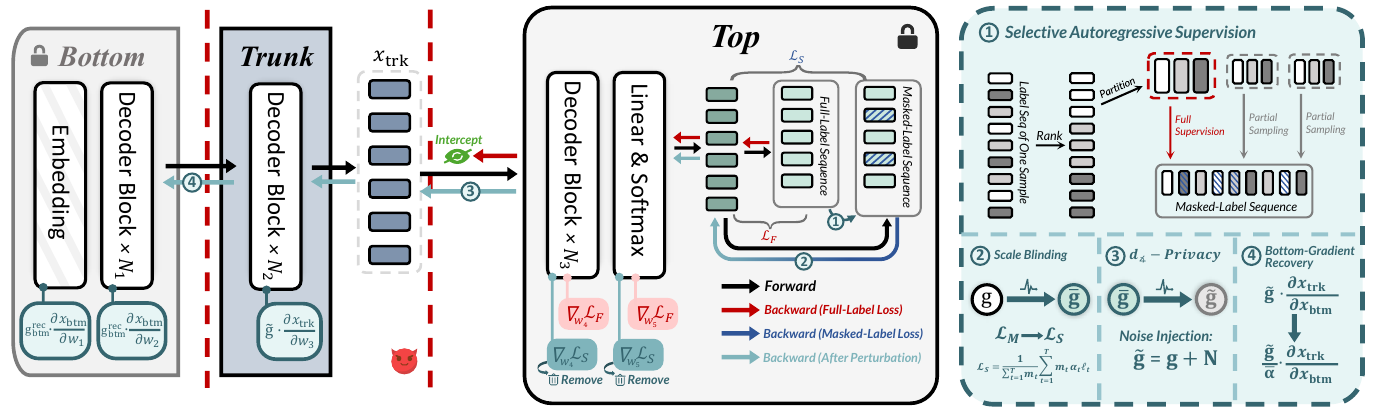}
    \caption{\textit{Illustration of Gradient Mirage against GMA-SL.} The private track (red) remains local, while the exposed track (blue-to-cyan) is protected before transmission. The numbered markers indicate the corresponding processing stages.}
    \label{fig:Defense_overview}
\end{figure*}

\noindent\textbf{Adversary's Capacities.} We endow the curious server with the following capabilities: $\boldsymbol{(i)}$ Interface observability. The curious server can monitor all cut-layer activations (i.e., clients' smashed representations) $x_{\mathrm{btm}}$ and the corresponding back-propagated gradients $\nabla x_{\mathrm{trk}}$ exchanged during the fine-tuning procedure. $\boldsymbol{(ii)}$ White-box access to the trunk. The adversary has full white-box knowledge of the server-side Trunk segment. $\boldsymbol{(iii)}$ Semi-white-box access to client-side segments, which will be explained below.
$\boldsymbol{(iv)}$ Flexible gradient-matching objective. The adversary may launch GMA-SL using any gradient-discrepancy measure as the matching objective, including cosine distance, $\ell_1$ distance, $\ell_2$ distance, or their weighted combinations (e.g., TAG).

\noindent\textbf{Semi-White-Box Access.} Following prior work on LLM-SL~\cite{SplitLLM}, we adopt a semi-white-box threat model. Concretely, the adversary does not have access to the victim’s exact per-step fine-tuned parameters or optimizer state, but is aware of the model architecture and pre-trained checkpoint. This access level—situated between the white-box and black-box extremes—is both realistic and widely adopted in LLM-SL, since pre-trained weights are often public or available to the server as a model provider. Consequently, the attacker can mount near-white-box attacks but still faces a \emph{fine-tuning gap} between the pre-trained proxy and the continually updated victim model.

\section{Gradient Mirage}
Gradient Mirage aims to induce three layers of inconsistency in GMA-SL while minimizing their impact on optimization (Fig.~\ref{fig:Defense_overview}). \emph{Objective Inconsistency} arises because the adversary assumes a full-label loss, whereas the exposed gradient is generated from a masked loss through Selective Autoregressive Supervision. \emph{Directional Inconsistency} remains because Directional Privatization randomizes the exposed gradient direction using the von Mises--Fisher mechanism. \emph{Scale Inconsistency} is introduced by Scale Blinding, making the gradient's natural magnitude unreliable even when its direction is approximately matched.
To further improve optimization compatibility, Gradient Mirage employs Dual-Track Backpropagation to preserve full-label learning at the Top segment and Bottom-Gradient Recovery to restore the effective gradient used for bottom optimization.

\subsection{Dual-Track Backpropagation}

In SL for autoregressive language modeling, the Top segment has access to full target labels, whereas lower segments receive only split-interface gradient, which may be exploited by GMA. This asymmetry suggests that full supervision should be confined to local learning within the Top segment, rather than directly reflected in the gradient exposed at the split interface.

To this end, we introduce two backward tracks with distinct roles: a \emph{Private Learning Track} and an \emph{Exposed Optimization Track}. For an input sequence, let $x_{\mathrm{trk}}$ denote the hidden representation sent from the trunk to the Top segment. We define a standard full-label autoregressive loss $\mathcal{L}_{F}$ over all target tokens and a masked loss $\mathcal{L}_{M}$ over a selected subset of tokens.
After a single forward pass produces logits $Z_{pre}$, the two tracks perform separate backward passes. In the Private Learning Track, $\mathcal{L}_{F}$ is used solely to optimize the Top-segment parameters, while its gradient with respect to $x_{\mathrm{trk}}$ is discarded. In the Exposed Optimization Track, $\mathcal{L}_{M}$ produces the interface gradient transmitted to the lower segments. This dual-track design preserves full-label learning at the Top segment while exposing only a controlled optimization signal across the split interface.

\subsection{Scale Blinding}
Building on the masked loss $\mathcal{L}_M$ defined above, we introduce
Scale Blinding to induce scale inconsistency by randomizing the
magnitudes of token-level gradient contributions. Specifically, we
replace the default unit weight of each supervised token with an
independently sampled coefficient $\alpha_t>0$, yielding the reweighted
masked loss
\begin{equation}
\mathcal{L}_{S}
=
\frac{1}{\sum_{t=1}^{T}m_t}
\sum_{t=1}^{T}m_t\alpha_t\ell_t,
\
\bar{\mathbf{g}}
=
\nabla_{x_{\mathrm{trk}}}\mathcal{L}_{S}.
\end{equation}
The released gradient therefore no longer reflects the natural token-wise scale induced by $\mathcal{L}_M$, hindering gradient matching and reconstruction. We parameterize the blinding distribution by its mean scale and relative token-wise variation, and select $\operatorname{Unif}[1500,2000]$ from a broad effective range; sensitivity analyses and scale-aware adaptive-attack results are provided in Appendix D.
This randomized reweighting avoids deterministic scale patterns, incurs negligible computational overhead, and preserves supervision at all selected positions while disrupting magnitude-sensitive matching objectives such as $\ell_1$, $\ell_2$, and their weighted combination.

\subsection{Selective Autoregressive Supervision}
In LLM-SL, the exposed split-interface gradient depends critically on the token positions included in the loss. To this end, we propose \textbf{Selective Autoregressive Supervision (SAS)}, summarized in Algorithm~\ref{alg:sas}, which supervises only a structured subset of token positions in each exposed backward pass. SAS preserves a useful optimization signal at the Top segment while inducing a mismatch between the true training gradient and the gradient assumed by an attacker.

For each sample, we construct the autoregressive target sequence by left-shifting the input label sequence, yielding $\mathbf{y} = (y_1, y_2, \ldots, y_T)$, where \(T\) is the number of supervised target tokens. During the first forward pass, the Top segment produces token-wise predictive distributions \(\{p_t\}_{t=1}^{T}\), from which we compute a token-level difficulty score based on predictive entropy $H_t = - \sum_{v \in \mathcal{V}} p_t(v)\log p_t(v)$, where \(\mathcal{V}\) denotes the vocabulary. Intuitively, a larger entropy means greater uncertainty about token \(y_t\), indicating that it is harder under the current model state.

We sort the sample’s token positions by entropy and partition them into three strata: \textbf{high-entropy}, \textbf{medium-entropy}, and \textbf{low-entropy}. Each stratum is then randomly and evenly divided into \(k\) disjoint subgroups, yielding \(3k\) entropy-aware subgroups in total, where \(k\) is a user-specified hyperparameter. We next construct \(k\) token groups by pairing one subgroup from each entropy stratum, so that each group contains a balanced mix of high-, medium-, and low-entropy tokens. Denoting these token groups by $\mathcal{G}_1, \mathcal{G}_2, \ldots, \mathcal{G}_k,$
we have
$
\mathcal{G}_i \cap \mathcal{G}_j = \emptyset \quad (i \neq j), \
\bigcup_{i=1}^{k} \mathcal{G}_i = \{1,2,\ldots,T\}.
$
Hence, the $k$ token groups are disjoint and jointly cover the entire label sequence.

At each training step, one token group is selected as the \textbf{full-supervision group}, and all its token positions are included in the masked loss. For the remaining $k-1$ token groups, we randomly sample a fixed proportion of token positions and include only the sampled subset in the loss computation. Thus, the exposed loss is neither standard full-label supervision nor naive random masking, but a structured, entropy-aware supervision pattern that varies across steps while maintaining coverage over the full label sequence across training.

\newcommand{\bluecomment}[1]{\textcolor{blue}{#1}}
\SetCommentSty{bluecomment}
\begin{algorithm}[t]
\caption{Selective Autoregressive Supervision}
\label{alg:sas}
\KwIn{Label sequence $\mathbf y=(y_1,\dots,y_T)$; predictive distributions $\{p_t\}_{t=1}^T$ from the first forward pass; group number $k$; sampling ratio $\rho$}
\KwOut{Supervision mask $\mathbf m=(m_1,\dots,m_T)$}

Construct autoregressive targets by left-shifting $\mathbf y$\;
\For{$t=1$ \KwTo $T$}{
    \tcp{\textit{Compute token entropy}}
    $H_t \leftarrow -\sum_{v\in\mathcal V} p_t(v)\log p_t(v)$\;
}
Sort token positions $\{1,\dots,T\}$ by $H_t$ in descending order\;
Partition the sorted positions into three strata:
$\mathcal H^{\mathrm{high}}, \mathcal H^{\mathrm{mid}}, \mathcal H^{\mathrm{low}}$\;
Randomly and evenly split each stratum into $k$ disjoint subgroups:
$\{\mathcal H^{\mathrm{high}}_i\}_{i=1}^k$,
$\{\mathcal H^{\mathrm{mid}}_i\}_{i=1}^k$,
$\{\mathcal H^{\mathrm{low}}_i\}_{i=1}^k$\;
\For{$i=1$ \KwTo $k$}{
    Form token group
    $
    \mathcal G_i \leftarrow
    \mathcal H^{\mathrm{high}}_i
    \cup
    \mathcal H^{\mathrm{mid}}_i
    \cup
    \mathcal H^{\mathrm{low}}_i
    $
\;}
Select one token group $\mathcal G_r$ as the full-supervision group\;
\tcp{\textit{Full supervision}}
Initialize $m_t \leftarrow 0$ for all $t=1,\dots,T$\;
\ForEach{$t\in \mathcal G_r$}{
    $m_t \leftarrow 1$
}
\For{$j=1$ \KwTo $k$}{
    \If{$j\neq r$}{
        Randomly sample $\mathcal S_j \subseteq \mathcal G_j$ with total sampling ratio $\rho$\;
        \tcp{\textit{Partial supervision}}
        \ForEach{$t\in \mathcal S_j$}{
            $m_t \leftarrow 1$
        }
    }
}
\Return{$\mathbf m$}\;
\end{algorithm}

Formally, let $m_t \in \{0,1\}$ denote the supervision mask for token position $t$. For a selected full-supervision group $\mathcal{G}_r$, we set $m_t = 1, \ \forall t \in \mathcal{G}_r$.
For each remaining group $\mathcal{G}_j$ with $j \neq r$, we independently sample a subset $\mathcal{S}_j \subseteq \mathcal{G}_j$ according to a total token sampling ratio $\rho \in [0,1]$, and set
$m_t = 1, \ \forall t \in \mathcal{S}_j, \ m_t = 0, \ \forall t \in \mathcal{G}_j \setminus \mathcal{S}_j.$
The masked autoregressive loss is then defined as
$
\mathcal{L}_M
=
\frac{1}{\sum_{t=1}^{T} m_t}
\sum_{t=1}^{T}
m_t \,\ell_t,\ \ell_t = -\log p_t(y_t).$

The design of SAS serves two complementary purposes. First, by ensuring that one token group is fully supervised at each step, the returned gradient still carries a sufficiently strong and coherent optimization signal. Second, by masking the remaining groups through entropy-aware partial sampling, the actual supervision pattern deviates from the attacker’s default assumption that the observed interface gradient corresponds to a standard full-label autoregressive loss over the complete target sequence.

The hyperparameter $k$ controls the granularity of the entropy-aware supervision and can be chosen according to the training schedule. Notably, $k$ is not required to equal the total number of training epochs. For example, even when fine-tuning lasts for 
$E=6$ epochs, one may set $k=3$, use the resulting three token groups for 3 epochs, and then recompute the token-wise predictive distributions and entropy-based grouping under the updated model state for the remaining 3 epochs.  In our LLM fine-tuning experiments, we use $E=3$ epochs and accordingly set $k=3$.

\subsection{Directional Privatization}
To introduce directional inconsistency, we adopt the vMF mechanism for directional perturbation. Rather than adding isotropic noise to the gradient, we perturb only its direction while preserving its magnitude, thereby obscuring direction while retaining optimization stability.

Given a nonzero interface gradient $\mathbf{g}\in\mathbb{R}^{d}$, let
$ \mu=\frac{\mathbf{g}}{\|\mathbf{g}\|_{2}}\in \mathbb{S}^{d-1}.$
We then sample a perturbed direction
$\tilde{\mathbf{\mu}}\sim \operatorname{VMF}(\mathbf{\mu},\kappa),$
and release
$\tilde{\mathbf{g}}=\|\mathbf{g}\|_{2}\tilde{\mathbf{\mu}},$ where the concentration parameter $\kappa\geq 0$ controls the
perturbation strength. Larger $\kappa$ concentrates
$\tilde{\mu}$ around $\mu$, whereas
$\kappa=0$ yields the uniform distribution on the hypersphere. The vMF mechanism satisfies Directional Metric Differential Privacy defined in Definition~\ref{def:1} under
the angular distance
$d_{\angle}(\mathbf{u},\mathbf{v})
=\arccos(\mathbf{u}^{\top}\mathbf{v})$. Here, the concentration parameter $\kappa$ determines the corresponding privacy budget $\varepsilon$.
\begin{definition}[\textbf{Directional Metric Differential Privacy}]
\label{def:1}
Let \((\mathbb{S}^{d-1}, d)\) be the unit hypersphere equipped with angular metric \(d_{\angle}\), and let \(\varepsilon > 0\). A randomized mechanism \(\mathcal{M}: \mathbb{S}^{d-1} \rightarrow \mathcal{P}(\mathbb{S}^{d-1})\) satisfies \(\varepsilon d\)-metric differential privacy if, for any two input directions \(\mathbf{u},\mathbf{v} \in \mathbb{S}^{d-1}\) and any measurable output set \(S \subseteq \mathbb{S}^{d-1}\), it holds that
\[\Pr[\mathcal{M}(\mathbf{u}) \in S]\leq
\exp\left(\varepsilon d_{\angle}(\mathbf{u},\mathbf{v})\right)
\Pr[\mathcal{M}(\mathbf{v}) \in S].\]
\end{definition}

Compared with additive-noise DP mechanisms~\cite{dp}, the vMF mechanism avoids the need for a user-specified clipping threshold $C$, which is often difficult to estimate and tune in practice. 
The vMF mechanism also preserves the gradient magnitude, which uniformly bounds amplitude distortion. 
Specifically, for $\mathbf{n}=\tilde{\mathbf{g}}-\mathbf{g}$, we have $\|\mathbf{g}\|_2/\|\mathbf{n}\|_2 \ge 1/2$ for every realization, with equality when $\tilde{\boldsymbol{\mu}}=-\boldsymbol{\mu}$.
Hence, it provides directional privacy while preserving sufficient optimization signal for stable training.

\begin{theorem}[\textbf{Lower Bound on the Amplitude Signal-to-Noise Ratio}]
\label{theo:2}
Let \(\mathbf{g}\in\mathbb{R}^{d}\) be a nonzero gradient and let
$\mathbf{\mu}
=
\frac{\mathbf{g}}{\|\mathbf{g}\|_2}.$
Suppose that the vMF mechanism samples a perturbed direction \(\tilde{\mathbf{\mu}}\in\mathbb{S}^{d-1}\) and releases $\tilde{\mathbf{g}}=\|\mathbf{g}\|_2\tilde{\mathbf{\mu}}.$

Define the perturbation noise as
$\mathbf{n}=\tilde{\mathbf{g}}-\mathbf{g},$
and amplitude signal-to-noise ratio as
$\mathrm{ASNR}=\frac{\|\mathbf{g}\|_2}{\|\mathbf{n}\|_2}.$
For any realization of \(\tilde{\mathbf{\mu}}\), the amplitude signal-to-noise ratio satisfies
$\mathrm{ASNR}\geq\frac{1}{2},$ where the case \(\tilde{\mathbf{\mu}}=\mathbf{\mu}\) is interpreted as \(\mathrm{ASNR}=+\infty\). The lower bound is tight and is attained when \(\tilde{\mathbf{\mu}}=-\mathbf{\mu}\).
\end{theorem}

\subsection{Bottom-Gradient Recovery}
Scale Blinding deliberately amplifies token-level gradient contributions, which may distort the gradient scale used for Bottom optimization. We therefore normalize the gradient reaching the Bottom segment by the expected blinding factor. Let
$
\bar{\alpha}
=
\mathbb{E}[\alpha_t],
$
and let $\tilde{\mathbf{g}}$ denote the protected interface gradient. For
$x_{\mathrm{trk}}=f(x_{\mathrm{btm}})$, the gradient propagated through the Trunk is
$\tilde{\mathbf{g}}_{\mathrm{btm}}=J_f(x_{\mathrm{btm}})^\top\tilde{\mathbf{g}}.$
Before updating the Bottom segment, we recover its effective scale as
$
\mathbf{g}_{\mathrm{btm}}^{\mathrm{rec}}=\frac{1}{\bar{\alpha}}
\tilde{\mathbf{g}}_{\mathrm{btm}}.
$

Since $\bar{\alpha}$ is the expected token weight, this normalization removes the average scale amplification introduced by Scale Blinding while retaining its token-wise randomness.

\section{Experiments}
\label{sec:5}

\subsection{Experimental Configuration}
\noindent\textbf{Models and Datasets.\footnote{We present representative results and key analyses in this section, with more comprehensive experiments and detailed analyses deferred to Appendix~\ref{sec:c5}.}}
We evaluate GMA-SL and its defenses on three LLMs:
Llama-2-7B~\cite{llama2}, Llama-3-8B~\cite{llama3}, and
DeepSeek-LLM-7B~\cite{deepseekllm}, using CodeAlpaca~\cite{codealpaca},
GSM8K~\cite{gsm8k}, and PIQA~\cite{piqa}. All models and attack
tensors use bfloat16. We assign six decoder layers to the Bottom
segment and five to the Top segment, with the remaining layers hosted
in the Trunk. 

\noindent\textbf{Evaluation Metrics.} We evaluate the data reconstruction performance of GMA-SL using the following metrics: $\boldsymbol{(i)}$ ROUGE-based F1 scores~\cite{rouge}, primarily ROUGE-L F1; $\boldsymbol{(ii)}$ Meteor~\cite{meteor}, as a complementary semantically aware metric. We use mean perplexity (PPL) to evaluate the fine-tuning performance of the LLMs.


\noindent\textbf{Attack Setting.} We perform attacks periodically during fine-tuning. Specifically, over a total of 600 fine-tuning steps, we launch one GMA-SL attack every 50 steps and evaluate the corresponding reconstruction metrics. We adopt BiSR(b)~\cite{SplitLLM}, a state-of-the-art GMA-SL paradigm, with the optimization objective:
\begin{equation}
\begin{aligned}
L(\tilde{y}) =\;& \beta \left\|\nabla_{x_{\mathrm{trk}}}\mathcal{L}(x_{\mathrm{trk}}, \tilde{y})
      - \nabla_{x_{\mathrm{trk}}}\mathcal{L}(x_{\mathrm{trk}}, y)\right\|_2 \\
+ (1-&\beta)\left\|\nabla_{x_{\mathrm{trk}}}\mathcal{L}(x_{\mathrm{trk}}, \tilde{y})
      - \nabla_{x_{\mathrm{trk}}}\mathcal{L}(x_{\mathrm{trk}}, y)\right\|_1.
\end{aligned}
\end{equation}
In our experiments, \(\beta\) is set to \(0.65\). In the ablation study, we also use cosine distance as the matching objective.

\subsection{Existing Privacy-Preserving Defenses}
\noindent\textbf{Gradient Pruning.} An intuitive defense against GMA-SL is Gradient Pruning (GP)~\cite{gradprune}, which forces the adversary to match against an incomplete gradient. 
Fig.~\ref{fig:gp}  presents the effect of GP against GMA-SL under different pruning rates (PRs), evaluated on Llama-3-8B. It can be seen that GP does not effectively defend against GMA-SL, and training becomes unstable when PR exceeds 90\%.

\noindent\textbf{Gradient Dropout.} Gradient Dropout (GD)~\cite{gradient_drop} transforms the shared gradient by randomly retaining a subset of coordinates and perturbing the rest.
Each gradient component is kept with probability \(p\) and rescaled by \(1/p\); otherwise, it is replaced with Gaussian noise sampled from \(\mathcal{N}(0,\sigma^2)\). The key idea is to preserve partial training signal while disrupting coordinate-wise gradient information, thereby hindering GMAs. Fig.~\ref{fig:gd} presents the effect of GD against GMA-SL under different drop rates and Gaussian noise standard deviations, evaluated on Llama-3-8B.

\begin{table*}[t]
\centering
\setlength{\tabcolsep}{1mm}

\begin{tabular}{c|c|ccc|ccc|ccc}
\toprule
\multicolumn{2}{c|}{Dataset}                                                                                           & \multicolumn{3}{c|}{CodeAlpaca} & \multicolumn{3}{c|}{PIQA}           & \multicolumn{3}{c}{GSM8K}           \\ \midrule
\multicolumn{1}{c|}{Model}       & Method     & RougeL-F  & Meteor  & PPL       & RougeL-F & Meteor   & PPL           & RougeL-F & Meteor   & PPL           \\ 
\midrule
\multicolumn{1}{c|}{\multirow{8}{*}{\begin{tabular}[c]{@{}c@{}}Llama3\\ -8B\end{tabular}}}       & GP                  & 1.00{\scriptsize $\pm$.00}  & .99{\scriptsize $\pm$.00} & 5.29{\scriptsize $\pm$.54}  & .96{\scriptsize $\pm$.01}  & .94{\scriptsize $\pm$.02}  & 11.50{\scriptsize $\pm$2.16}    & 1.00{\scriptsize $\pm$.00} & 1.00{\scriptsize $\pm$.01} & 12.23{\scriptsize $\pm$1.13}    \\
\multicolumn{1}{c|}{}                                                                            & GD                  & .94{\scriptsize $\pm$.01}   & .94{\scriptsize $\pm$.01} & 4.78{\scriptsize $\pm$.10}  & .95{\scriptsize $\pm$.01}  & .91{\scriptsize $\pm$.00}  & 7.76{\scriptsize $\pm$.03}      & .95{\scriptsize $\pm$.01}  & .94{\scriptsize $\pm$.01}  & 12.48{\scriptsize $\pm$.97}     \\
\multicolumn{1}{c|}{}   & LDP{\scriptsize (10240)}          & \underline{.56}{\scriptsize $\pm$.03}   & \underline{.57}{\scriptsize $\pm$.02} & 4.67{\scriptsize $\pm$.00}  & \underline{.80}{\scriptsize $\pm$.03}  & \underline{.76}{\scriptsize $\pm$.05}  & 7.74{\scriptsize $\pm$.06}      & .90{\scriptsize $\pm$.01}  & .92{\scriptsize $\pm$.01}  & 11.04{\scriptsize $\pm$.03}     \\
\multicolumn{1}{c|}{}                                                                            & LDP{\scriptsize (12288)}      & .66{\scriptsize $\pm$.03}   & .66{\scriptsize $\pm$.04} & 4.61{\scriptsize $\pm$.04}  & .81{\scriptsize $\pm$.03}  & .77{\scriptsize $\pm$.05}  & \underline{7.20}{\scriptsize $\pm$.02}      & \underline{.89}{\scriptsize $\pm$.02}  & \underline{.89}{\scriptsize $\pm$.02}  & 11.12{\scriptsize $\pm$.04}     \\
\multicolumn{1}{c|}{}     & Top-Only            & -         & -       & 4.83{\scriptsize $\pm$.01}  & -        & -        & 8.87{\scriptsize $\pm$.02}      & -        & -        & 12.00{\scriptsize $\pm$.01}     \\
\multicolumn{1}{c|}{}    & Standard            & 1.00{\scriptsize $\pm$.00}  & .99{\scriptsize $\pm$.00} & \underline{4.44}{\scriptsize $\pm$.00}  & .99{\scriptsize $\pm$.00}  & .97{\scriptsize $\pm$.01}  & 7.40{\scriptsize $\pm$.01}      & 1.00{\scriptsize $\pm$.00} & 1.00{\scriptsize $\pm$.00} & \underline{10.59}{\scriptsize $\pm$.00}     \\

\cmidrule{2-11}

\multicolumn{1}{c|}{}     & \textbf{Ours{\scriptsize (512)}}  & \textbf{.20}{\scriptsize $\pm$.00}   & \textbf{.23}{\scriptsize $\pm$.00} & 4.45{\scriptsize $\pm$.00}  & \textbf{.21}{\scriptsize $\pm$.02}  & \textbf{.14}{\scriptsize $\pm$.01}  & 7.40{\scriptsize $\pm$.02}      & \textbf{.11}{\scriptsize $\pm$.01}  & \textbf{.07}{\scriptsize $\pm$.01}  & 10.64{\scriptsize $\pm$.03}     \\
\multicolumn{1}{c|}{}       & \textbf{Ours{\scriptsize (1024)}} & \textbf{.35}{\scriptsize $\pm$.02}   & \textbf{.34}{\scriptsize $\pm$.02} & \textbf{4.42}{\scriptsize $\pm$.01}  & \textbf{.45}{\scriptsize $\pm$.02}  & \textbf{.31}{\scriptsize $\pm$.03}  & \textbf{7.19}{\scriptsize $\pm$.03}      & \textbf{.36}{\scriptsize $\pm$.01}  & \textbf{.21}{\scriptsize $\pm$.01}  & \textbf{10.50}{\scriptsize $\pm$.01}     \\ 
\midrule

\multicolumn{1}{c|}{\multirow{8}{*}{\begin{tabular}[c]{@{}c@{}}Llama2\\ -7B\end{tabular}}}       & GP                  & .99{\scriptsize $\pm$.00}   & .99{\scriptsize $\pm$.01} & 5.00{\scriptsize $\pm$.09}  & 1.00{\scriptsize $\pm$.00} & 1.00{\scriptsize $\pm$.01} & 6.50{\scriptsize $\pm$.62}      & .98{\scriptsize $\pm$.01}  & .98{\scriptsize $\pm$.00}  & 10.66{\scriptsize $\pm$.23}     \\
\multicolumn{1}{c|}{}                                                                            & GD                  & .90{\scriptsize $\pm$.02}   & .89{\scriptsize $\pm$.02} & 6.98{\scriptsize $\pm$1.58} & .81{\scriptsize $\pm$.02}  & .80{\scriptsize $\pm$.02}  & 7.76{\scriptsize $\pm$1.77}     & .90{\scriptsize $\pm$.01}  & .92{\scriptsize $\pm$.00}  & 107.08{\scriptsize $\pm$163.00} \\
\multicolumn{1}{c|}{}                                                                            & LDP{\scriptsize (10240)}     & \underline{.78}{\scriptsize $\pm$.03}   & \underline{.74}{\scriptsize $\pm$.04} & 4.80{\scriptsize $\pm$.08}  & \underline{.79}{\scriptsize $\pm$.04}  & \underline{.74}{\scriptsize $\pm$.04}  & 6.30{\scriptsize $\pm$.66}      & \underline{.38}{\scriptsize $\pm$.02}  & \underline{.33}{\scriptsize $\pm$.01}  & 10.49{\scriptsize $\pm$.15}     \\
\multicolumn{1}{c|}{}                                                                            & LDP{\scriptsize (12288)}    & .82{\scriptsize $\pm$.04}   & .80{\scriptsize $\pm$.04} & 4.85{\scriptsize $\pm$.13}  & .83{\scriptsize $\pm$.02}  & .79{\scriptsize $\pm$.02}  & 5.85{\scriptsize $\pm$.05}      & .44{\scriptsize $\pm$.03}  & .40{\scriptsize $\pm$.03}  & 10.48{\scriptsize $\pm$.14}     \\
\multicolumn{1}{c|}{}                                                                            & Top-Only            & -         & -       & 4.96{\scriptsize $\pm$.04}  & -        & -        & 6.27{\scriptsize $\pm$.06}      & -        & -        & 11.20{\scriptsize $\pm$.28}     \\
\multicolumn{1}{c|}{}       & Standard            & .99{\scriptsize $\pm$.00}   & .98{\scriptsize $\pm$.01} & \textbf{4.73}{\scriptsize $\pm$.02}  & .99{\scriptsize $\pm$.01}  & .99{\scriptsize $\pm$.01}  & \textbf{5.29}{\scriptsize $\pm$.10}      & .98{\scriptsize $\pm$.00}  & .98{\scriptsize $\pm$.00}  & \textbf{9.71}{\scriptsize $\pm$.14}      \\

\cmidrule{2-11}

\multicolumn{1}{c|}{}                                                                            & \textbf{Ours{\scriptsize (512)}}  & \textbf{.10}{\scriptsize $\pm$.01}   & \textbf{.11}{\scriptsize $\pm$.01} & 4.82{\scriptsize $\pm$.10}  & \textbf{.05}{\scriptsize $\pm$.02}  & \textbf{.05}{\scriptsize $\pm$.00}  & 5.63{\scriptsize $\pm$.13}      & \textbf{.03}{\scriptsize $\pm$.01}  & \textbf{.04}{\scriptsize $\pm$.01}  & 10.43{\scriptsize $\pm$.16}     \\
\multicolumn{1}{c|}{}                                                                            & \textbf{Ours{\scriptsize (1024)}} & \textbf{.17}{\scriptsize $\pm$.02}   & \textbf{.15}{\scriptsize $\pm$.01} & \underline{4.77}{\scriptsize $\pm$.05}  & \textbf{.10}{\scriptsize $\pm$.01}  & \textbf{.09}{\scriptsize $\pm$.01}  & \underline{5.56}{\scriptsize $\pm$.23}      & \textbf{.08}{\scriptsize $\pm$.01}  & \textbf{.07}{\scriptsize $\pm$.01}  & \underline{10.14}{\scriptsize $\pm$.17}     \\ 
\midrule
\multicolumn{1}{c|}{\multirow{8}{*}{\begin{tabular}[c]{@{}c@{}}DeepSeek\\ -LLM-7B\end{tabular}}} & GP                  & .99{\scriptsize $\pm$.00}   & .97{\scriptsize $\pm$.00} & 5.39{\scriptsize $\pm$.05}  & .94{\scriptsize $\pm$.02}  & .93{\scriptsize $\pm$.02}  & 14.39{\scriptsize $\pm$2.09}    & .79{\scriptsize $\pm$.01}  & .75{\scriptsize $\pm$.03}  & 39.11{\scriptsize $\pm$35.00}   \\
\multicolumn{1}{c|}{}                         & GD                  & .68{\scriptsize $\pm$.04}   & .68{\scriptsize $\pm$.04} & 6.33{\scriptsize $\pm$.33}  & .84{\scriptsize $\pm$.01}  & .82{\scriptsize $\pm$.03}  & 301.78{\scriptsize $\pm$415.97} & .84{\scriptsize $\pm$.02}  & .85{\scriptsize $\pm$.03}  & 16.10{\scriptsize $\pm$3.46}    \\
\multicolumn{1}{c|}{}     & LDP{\scriptsize (8192)}     & \underline{.60}{\scriptsize $\pm$.02}   & \underline{.54}{\scriptsize $\pm$.03} & 5.05{\scriptsize $\pm$.05}  & \underline{.82}{\scriptsize $\pm$.02}  & \underline{.76}{\scriptsize $\pm$.02}  & 6.71{\scriptsize $\pm$.24}      & \underline{.31}{\scriptsize $\pm$.02} & \underline{.21}{\scriptsize $\pm$.03} & 12.13{\scriptsize $\pm$.12}     \\
\multicolumn{1}{c|}{}   & LDP{\scriptsize (10240)}   & .74{\scriptsize $\pm$.02}   & .71{\scriptsize $\pm$.03} & \underline{5.00}{\scriptsize $\pm$.10}  & .85{\scriptsize $\pm$.01}  & .71{\scriptsize $\pm$.02}  & 6.72{\scriptsize $\pm$.27}      & .34{\scriptsize $\pm$.01} & .24{\scriptsize $\pm$.03} & 12.13{\scriptsize $\pm$.12}     \\
\multicolumn{1}{c|}{}              & Top-Only            & -         & -       & 5.44{\scriptsize $\pm$.12}  & -        & -        & 7.44{\scriptsize $\pm$.13}      & -        & -        & 13.72{\scriptsize $\pm$.11}     \\
\multicolumn{1}{c|}{}                                                                            & Standard            & 1.00{\scriptsize $\pm$.00}  & .98{\scriptsize $\pm$.00} & \textbf{4.38}{\scriptsize $\pm$.00}  & 1.00{\scriptsize $\pm$.00} & .97{\scriptsize $\pm$.01}  & \textbf{5.79}{\scriptsize $\pm$.01}      & .98{\scriptsize $\pm$.00}  & .98{\scriptsize $\pm$.00}  & \textbf{10.01}{\scriptsize $\pm$.00}     \\

\cmidrule{2-11}

\multicolumn{1}{c|}{}                                                                            & \textbf{Ours{\scriptsize (512)}}  & \textbf{.07}{\scriptsize $\pm$.02}   & \textbf{.05}{\scriptsize $\pm$.02} & 5.30{\scriptsize $\pm$.22}  & \textbf{.07}{\scriptsize $\pm$.01}  & \textbf{.04}{\scriptsize $\pm$.00}  & 6.67{\scriptsize $\pm$.16}      & \textbf{.03}{\scriptsize $\pm$.00}  & \textbf{.01}{\scriptsize $\pm$.00}  & \underline{11.99}{\scriptsize $\pm$.11}     \\
\multicolumn{1}{c|}{}                                                                            & \textbf{Ours{\scriptsize (1024)}} & \textbf{.11}{\scriptsize $\pm$.01}   & \textbf{.06}{\scriptsize $\pm$.01} & 5.15{\scriptsize $\pm$.14}  & .\textbf{14}{\scriptsize $\pm$.01}  & \textbf{.08}{\scriptsize $\pm$.01}  & \underline{6.45}{\scriptsize $\pm$.02}      & \textbf{.07}{\scriptsize $\pm$.01}  & \textbf{.03}{\scriptsize $\pm$.00}  & 12.04{\scriptsize $\pm$.29}     \\ 
\bottomrule
\end{tabular}
\caption{\textit{Comparison between Gradient Mirage and existing
privacy-preserving methods across three models.}}
\label{tab:big}

\end{table*}

\noindent\textbf{Gradient Sequence Local DP.} Unlike DP-SGD~\cite{dp-sgd}, the object to be released and protected in our setting is the gradient sequence $G \in \mathbb{R}^{L \times D}$. Inspired by DP-Forward~\cite{DP-Forward}, which enforces privacy in the forward pass by clipping and perturbing client-side intermediate activations before transmission, we extend \textit{Sequence Local DP (SeqLDP)}~\cite{DP-Forward} to \textit{Gradient Sequence Local DP (GradSeq-LDP)}.
\begin{definition}[\textbf{GradSeq-LDP}]
For $\varepsilon \ge 0$, $0 \le \delta \le 1$, $\mathcal{M}$ satisfies
$(\varepsilon,\delta)$-GradSeq-LDP, if $\forall G, G'\in \mathbb{R}^{L \times D}$, and any
possible output subset $\mathcal{O}$,
\[
\Pr[\mathcal{M}(G)\in \mathcal{O}]
\le
e^{\varepsilon}\Pr[\mathcal{M}(G')\in \mathcal{O}] + \delta.
\]
\end{definition}
We implement \textit{GradSeq-LDP} using the analytic Gaussian mechanism (aGM)~\cite{aGM}. \textit{GradSeq-LDP} enforces privacy protection at an earlier stage than DP-SGD. Consequently, applying such a strong privacy mechanism at this location may substantially affect training.
As shown in Tab.~\ref{tab:Dp-Forward}, a small \(\boldsymbol{\varepsilon}\) indeed provides strong defense effectiveness, but leads to highly unstable training. Fig.~\ref{fig:DP-Forward} shows that fine-tuning under small privacy budgets is highly unstable and underperforms Top-Only Training. Even when a particular run appears stable, the training can become unstable after changing the random seed, as illustrated in Fig.~\ref{fig:not_stable}.

\begin{figure}[t]
    \centering
    \includegraphics[width=0.98\linewidth]{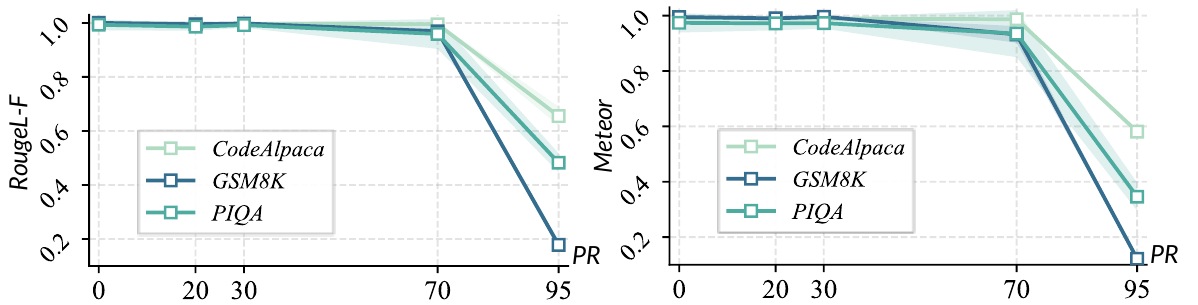}
    \caption{\textit{Effect of GP against GMA-SL under different PRs.} }
    \label{fig:gp}
\end{figure}

To quantify the perturbation imposed by GradSeq-LDP on the exposed gradients, we employ four complementary gradient utility metrics: Amplitude Signal-to-Noise Ratio (ASNR), ASNR@10\%, Recall@10\% (R@10\%), and Jaccard@10\% (J@10\%), which characterize the preservation of gradient magnitude and salient coordinates from complementary perspectives. Detailed definitions are provided in Appendix~\ref{sec:Appendix_C1}. As shown in Fig.~\ref{fig:SNR}, even with an extremely large privacy budget of $\varepsilon=12288$, the privatized gradients remain heavily dominated by injected noise, exhibiting limited preservation of the original gradient signal. Despite such severe gradient perturbation, GradSeq-LDP still provides insufficient protection against GMA-SL, highlighting its unfavorable privacy--utility trade-off in this setting.

\begin{figure}[t]
    \centering
    \includegraphics[width=0.98\linewidth]{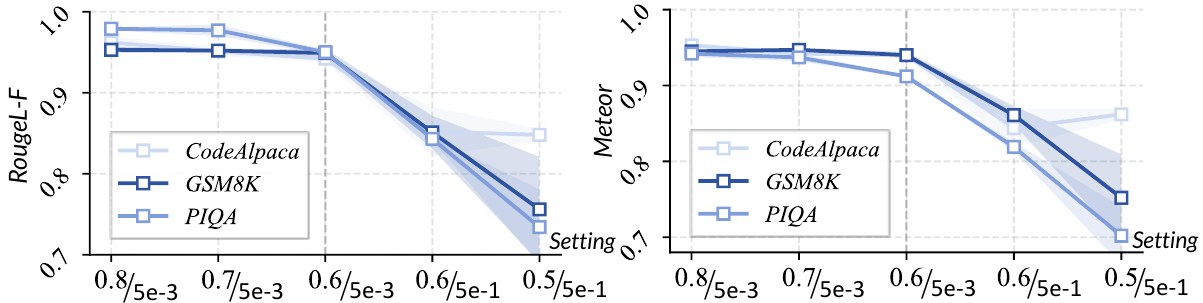}
    \caption{\textit{Effect of GD against GMA-SL under different drop rates and Gaussian noise standard deviations.} }
    \label{fig:gd}
\end{figure}

\begin{figure*}[t]
    \centering
    \includegraphics[width=0.98\linewidth]{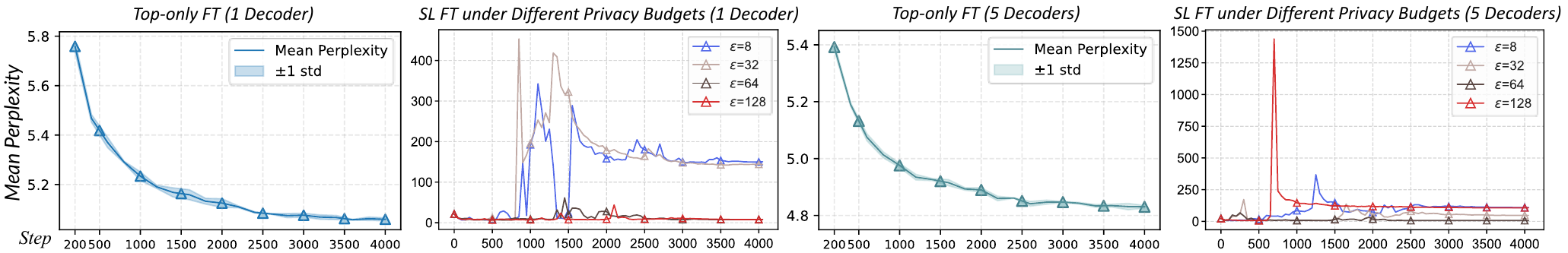}
    \caption{\textit{Comparison between Top-only FT and SL FT under GradSeq-LDP.} The left two panels correspond to 1-decoder and the right two to 5-decoder. Experiments were conducted on Llama3-8B using the CodeAlpaca dataset.}
    \label{fig:DP-Forward}
\end{figure*}

\begin{figure*}[t]
    \centering
    \includegraphics[width=0.98\linewidth]{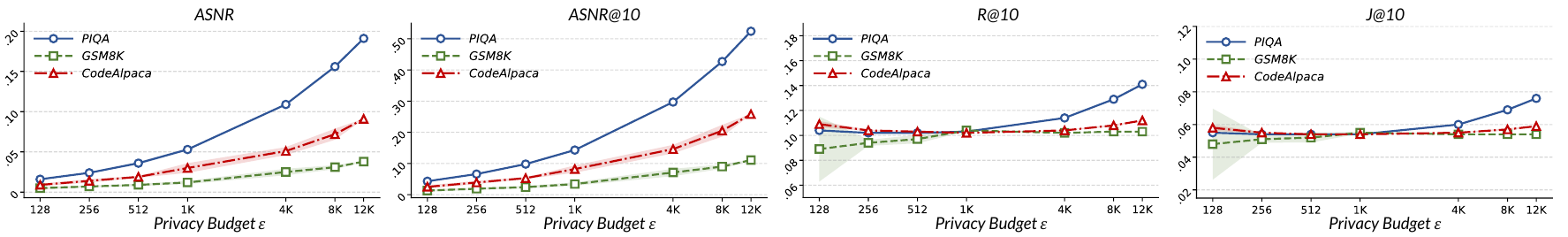}
    \caption{\textit{Gradient utility of GradSeq-LDP under different privacy budgets across three datasets.} Even at large \(\varepsilon\), the privatized gradients retain limited signal strength and salient-coordinate information.}
    \label{fig:SNR}
\end{figure*}

\begin{figure}[t]
    \centering
    \includegraphics[width=0.98\linewidth]{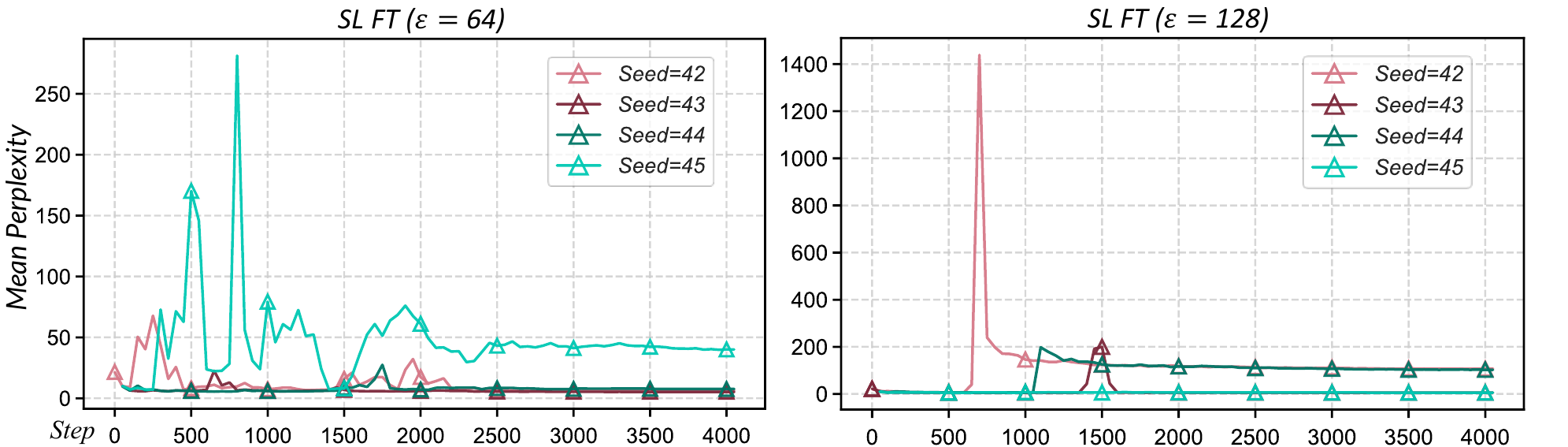}
    \caption{\textit{Seed sensitivity of SL FT across random seeds.} }
    \label{fig:not_stable}
\end{figure}

\subsection{Benchmarking Gradient Mirage}
To fairly compare defense effectiveness against BiSR(b), we tune all privacy-preserving methods to achieve comparable fine-tuning performance. Results are reported in Tab.~\ref{tab:big}. For GP, we use a PR of \(0.7\) for the Llama-series models and \(0.8\) for DeepSeek to obtain fine-tuning performance comparable to that of Gradient Mirage.  For GD, we set the retention probability to \(p=0.6\) and the Gaussian noise standard deviation to \(5\times10^{-3}\). For \textit{GradSeq-LDP}, the clipping coefficients \(C\) are set to \(0.12\), \(0.02\), and \(0.01\) for Llama-3, Llama-2, and DeepSeek, respectively, while the corresponding privacy budgets \(\varepsilon\) are specified in the tables. For Gradient Mirage, we set the SAS sampling ratio to \(\rho=0.6\) and use the directional metric privacy budget \(\varepsilon \in \{512, 1024\}\) induced by the vMF mechanism for all models. All experiments are conducted with a batch size of \(8\) and repeated using three random seeds.

\subsection{Analysis of Scale Blinding Parameterization}
\label{sec:5.4}
\noindent\textbf{Parameterization.} To disentangle the effect of the absolute scaling magnitude from token-wise randomness, we parameterize the blinding coefficient as
\begin{equation}
\alpha_t = m(1 + u_t),  u_t \sim \mathrm{Unif}[-r, r],
\end{equation}
where \(m = \mathbb{E}[\alpha_t]\) controls the mean amplitude shift and \(r\) controls the relative token-wise variation. Our default choice \(\alpha_t \sim \mathrm{Unif}[1500, 2000]\) corresponds to \(m = 1750\) and \(r = 1/7\), with a coefficient of variation of approximately \(0.082\).

\noindent\textbf{Scale-Aware Adaptive Attack.}
We additionally consider an attacker that jointly estimates an optimal nuisance scale during reconstruction. 
Specifically, the attacker optimizes both the dummy label $y_{\mathrm{dum}}$ and a learnable raw-scale parameter $s_{\mathrm{raw}}$ using the following scale-aware gradient-matching objective:
$
    \min_{y_{\mathrm{dum}},\,s_{\mathrm{raw}}}
    \left\|
        \mathbf{g}_{\mathrm{release}}
        - s\,\mathbf{g}(y_{\mathrm{dum}})
    \right\|_2^2, \ s = \operatorname{softplus}(s_{\mathrm{raw}}) + \epsilon,
$
where $\mathbf{g}_{\mathrm{release}}$ denotes the released gradient, $\mathbf{g}(y_{\mathrm{dum}})$ denotes the dummy gradient generated from $y_{\mathrm{dum}}$, and $\epsilon_{\mathrm{num}}=10^{-8}$ is a small constant introduced for numerical stability. The softplus parameterization ensures that the estimated scale remains strictly positive throughout optimization. Both $y_{\mathrm{dum}}$ and $s_{\mathrm{raw}}$ are jointly updated during reconstruction.

\noindent\textbf{Effect of the Mean Scale $m$.}
To study the impact of the mean amplitude shift, we fix the relative token-wise variation as \(r=1/7\) and evaluate different mean scales:
\(m\in\{100,500,1000,1500,1750,2500,4000\}\).
The corresponding defense effectiveness and fine-tuning performance are reported in Tab.~\ref{tab_appendix:scale_m} and Fig.~\ref{fig_Appendix:scale_m}, respectively. Tab.~\ref{tab_appendix:scale_m} reports ROUGE-L-F as the reconstruction
metric. Throughout the Scale Blinding experiments in section~\ref{sec:5.4}, the directional DP budget is fixed at $\varepsilon$ = 512.

As the mean scale increases, the reconstruction quality decreases rapidly in the low-scale regime and gradually stabilizes, while the fine-tuning utility remains largely unaffected across different settings. This indicates that increasing the mean scale effectively disrupts magnitude-sensitive gradient matching, but the privacy gain becomes less significant once the scale reaches a sufficiently large range.

Based on these observations, we choose \(m=1750\) (corresponding to \(\alpha_t\sim\mathrm{Unif}[1500,2000]\)) as our default setting. Although this value is not necessarily the globally optimal choice, it provides a favorable privacy-utility trade-off while avoiding unnecessarily large gradient scaling.

\begin{table*}[t]
\centering
\resizebox{0.86\textwidth}{!}{%
\begin{tabular}{c|ccccccc}
\toprule
\multirow{2}{*}{\begin{tabular}[c]{@{}c@{}}Matching\\ Objective\end{tabular}} & \multicolumn{7}{c}{$m$}                                                             \\ \cmidrule{2-8} 
     & 100       & 500       & 1000      & 1500      & 1750      & 2500      & 4000      \\ 
\midrule
TAG       & .382±.026 & .246±.018 & .211±.012 & .195±.016 & .202±.002 & .170±.002 & .168±.010 \\
\multicolumn{1}{l|}{Cosine Distance}        & .089±.008 & .088±.006 & .087±.008 & .090±.003 & .094±.005 & .100±.004 & .088±.007 \\
Adaptive TAG                   & .360±.024 & .248±.018 & .213±.010 & .197±.021 & .208±.016 & .187±.005 & .184±.013 \\ \bottomrule
\end{tabular}

}
\caption{\textit{Effect of the mean scale \(m\) on reconstruction quality under different matching objectives.}}
\label{tab_appendix:scale_m}

\end{table*}
\begin{table*}[t]
\centering
\resizebox{0.76\textwidth}{!}{%
\begin{tabular}{c|ccccccc}
\toprule
\multirow{2}{*}{$r$} & \multicolumn{7}{c}{$m$}                                                             \\ 
\cmidrule{2-8} 
                   & 100       & 500       & 1000      & 1500      & 1750      & 2500      & 4000      \\
\midrule
1/5                & .341±.021 & .212±.014 & .191±.012 & .187±.012 & .183±.011 & .143±.013 & .133±.014 \\
1/7                & .360±.024 & .248±.018 & .213±.010 & .197±.021 & .208±.016 & .187±.005 & .184±.013 \\
1/10               & .391±.017 & .250±.011 & .219±.013 & .199±.023 & .208±.007 & .186±.011 & .189±.009 \\
1/25               & .397±.024 & .258±.017 & .221±.023 & .237±.015 & .219±.010 & .194±.010 & .193±.010 \\
1/50               & .404±.024 & .263±.017 & .227±.016 & .232±.022 & .239±.008 & .198±.005 & .204±.009 \\ 

\bottomrule

\end{tabular}

}

\caption{\textit{Effect of the relative token-wise variation \(r\) on reconstruction quality under different mean scales \(m\).}}
\label{tab_appendix:scale_r}

\end{table*}
\begin{table}[t]
\centering
\resizebox{0.78\columnwidth}{!}{%
\begin{tabular}{c|ccc}
\toprule
\multirow{2}{*}{BiSR(b)} & \multicolumn{3}{c}{CodeAlpaca}  \\ \cmidrule{2-4} 
        & RougeL-F & Rouge2-F & Meteor  \\ \midrule
\(\boldsymbol{\varepsilon}\)=128   & .02±.01   & .00±.00   & .01±.00 \\
\(\boldsymbol{\varepsilon}\)=256   & .03±.01   & .00±.00   & .02±.01 \\
\(\boldsymbol{\varepsilon}\)=1024  & .12±.05   & .05±.03   & .10±.05 \\
\(\boldsymbol{\varepsilon}\)=8192  & .62±.07   & .48±.09   & .62±.08 \\
\(\boldsymbol{\varepsilon}\)=10240 & .56±.03   & .41±.02   & .57±.02 \\
\(\boldsymbol{\varepsilon}\)=12288 & .66±.03   & .52±.04   & .66±.04 \\ \bottomrule
\end{tabular}
}
\caption{\textit{Defense performance of \textit{GradSeq-LDP} against GMA-SL under different privacy budgets.}}
\label{tab:Dp-Forward}
\end{table}

\begin{table}[t]

\centering
\resizebox{0.9\columnwidth}{!}{%
\begin{tabular}{c|ccc}
\toprule
Setting            & RougeL-F  & Rouge2-F  & Meteor    \\ 
\midrule
w/o Scale Blinding & .71±.01 & .51±.02 & .69±.01 \\
w/ Scale Blinding  & .35±.02 & .18±.02 & .34±.02 \\ 
\bottomrule
\end{tabular}
}
\caption{\textit{Impact of Scale Blinding.}}
\label{tab:Ablation_scale}
\end{table}
\begin{table}[t]
\centering
\setlength{\tabcolsep}{1mm}
\resizebox{0.84\columnwidth}{!}{%
\begin{tabular}{l|ccc}
\toprule
Setting             & RougeL-F  & Rouge2-F  & Meteor    \\ 
\midrule
Scale Blinding-Only & .55±.11 & .25±.04 & .43±.11 \\
$\rho$=0.8               & .42±.09 & .11±.04 & .28±.09 \\
$\rho$=0.7               & .41±.06 & .11±.03 & .28±.04 \\
$\rho$=0.6               & .37±.04 & .09±.02 & .25±.03 \\
$\rho$=0.6,\(\boldsymbol{\varepsilon}\)=8192        & .29±.02 & .05±.01 & .19±.02 \\
$\rho$=0.6,\(\boldsymbol{\varepsilon}\)=1024        & .17±.01 & .01±.00 & .11±.01 \\
$\rho$=0.6,\(\boldsymbol{\varepsilon}\)=512         & .09±.01 & .00±.00 & .07±.01 \\ 
\bottomrule
\end{tabular}
}

\caption{\textit{Impact of SAS and Directional Privatization.}}
\label{tab:ablation_dir}

\end{table}

\noindent\textbf{Effect of Relative Token-wise Variation $r$.}
We further investigate the impact of the relative token-wise variation \(r\), which controls the randomness introduced into the token-level scaling coefficients. We evaluate different values of \(r\) under multiple mean scales \(m\), and the reconstruction results are summarized in Tab.~\ref{tab_appendix:scale_r}. The experiments are conducted on the CodeAlpaca dataset using the Llama-3-8B model. ROUGE-L-F is reported as the reconstruction metric.

When \(r\) is close to zero, the scaling coefficients become concentrated around a similar value, causing Scale Blinding to degenerate toward deterministic global rescaling. Such a transformation can be partially compensated by an adaptive attacker and provides limited disruption to the relative magnitude pattern among token-level gradients.

As shown in Table~\ref{tab_appendix:scale_r}, introducing token-wise variation consistently improves the effectiveness of Scale Blinding compared with small variation settings. However, the benefit does not increase monotonically with \(r\). Excessively large variation may introduce overly heterogeneous gradient contributions across tokens, which can increase optimization disturbance without providing proportional privacy gains.

Considering both reconstruction resistance and optimization stability, we choose \(r=1/7\) as the default setting. This value introduces sufficient token-wise randomness while avoiding excessive perturbation of the gradient structure.

\subsection{Ablation Study}

\noindent\textbf{Impact of SAS and Directional Privatization.}
We use cosine distance as the matching objective to isolate SAS and Directional Privatization. As shown in Tab.~\ref{tab:ablation_dir}, both SAS and Directional Privatization substantially reduce reconstruction performance by perturbing the positional and directional information of the exposed gradient, respectively. The configuration of Scale Blinding is kept fixed throughout this experiment.

\noindent\textbf{Impact of Scale Blinding.}
We evaluate Scale Blinding against BiSR(b), which adopts the magnitude-sensitive objective used by TAG. As shown in Tab.~\ref{tab:Ablation_scale}, removing Scale Blinding increases ROUGE-L F1 from 0.35 to 0.71, indicating that SAS and Directional Privatization are insufficient. The vMF mechanism uses a privacy budget of $\varepsilon=1024$, while SAS is configured with a total token ratio of $\rho=0.6$.

\noindent\textbf{Impact of Dual-Track Backpropagation.}
As shown in Tab.~\ref{tab:big_ablation}, removing Dual-Track Backpropagation consistently increases PPL across all datasets. This confirms that retaining full-label supervision in the \textit{private learning track} is important for preserving fine-tuning utility.

\begin{figure*}[t]
    \centering
    \includegraphics[width=0.84\linewidth]{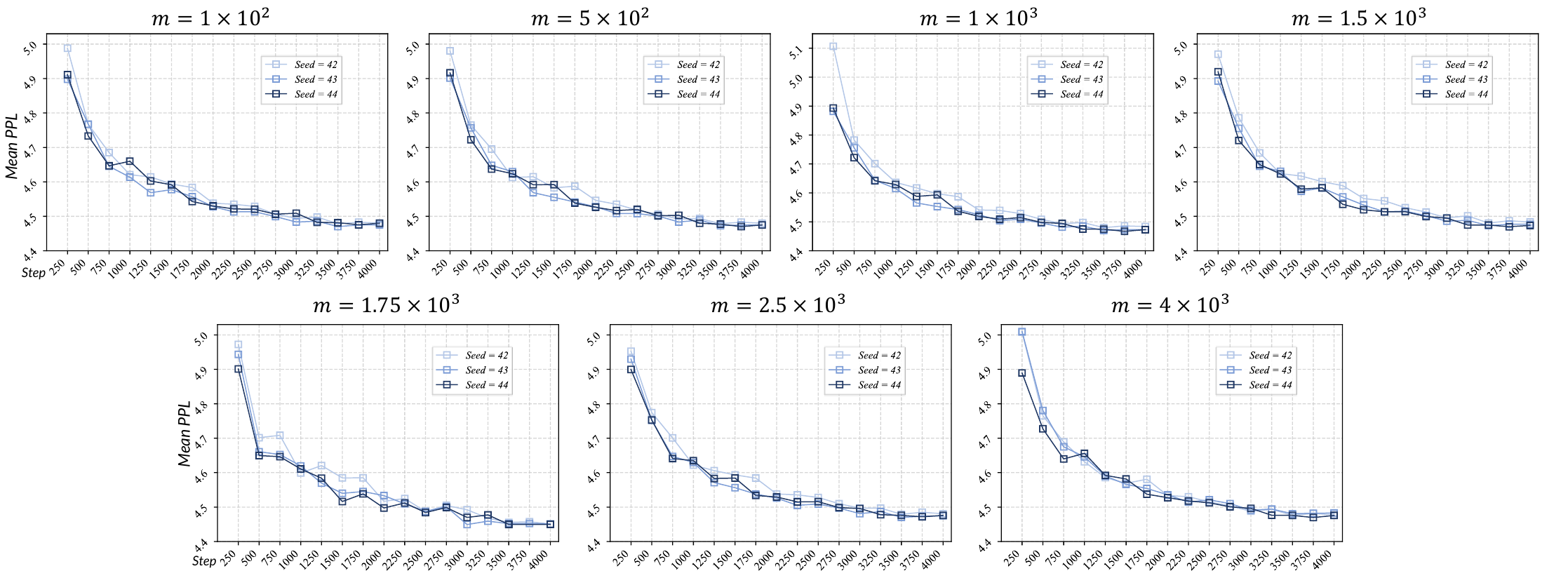}
    \caption{\textit{Fine-tuning performance under different mean scales \(m\) with fixed relative token-wise variation \(r=1/7\). The results are obtained on Llama-3-8B with the CodeAlpaca dataset.} }
    \label{fig_Appendix:scale_m}
    \vspace{-3pt}
\end{figure*}

\begin{table*}[t]
\centering

\begin{tabular}{cc|cc|cc|cc}
\toprule
\multirow{2}{*}{Dataset} & \multirow{2}{*}{\begin{tabular}[c]{@{}c@{}}Standard\\ PPL\end{tabular}} & \multicolumn{2}{c|}{Dual-Track BP ({\color{red}\ding{55}})} & \multicolumn{2}{c|}{Bottom-Gradient Rec ({\color{red}\ding{55}})} & \multicolumn{2}{c}{Trunk Training ({\color{red}\ding{55}})} \\  \cmidrule{3-8}
              &          & PPL                & $\Delta (\%)$         & PPL             & $\Delta (\%)$       & PPL         & $\Delta (\%)$         \\ \midrule
\multicolumn{1}{c|}{CodeAlpaca}     & 4.42±.01    & 4.51±.05   & {\color[HTML]{FE0000} 2.04$\uparrow$}          & 4.51±.01        & {\color[HTML]{FE0000} 2.04$\uparrow$}        & 4.48±.05           & {\color[HTML]{FE0000} 1.36$\uparrow$}          \\

\multicolumn{1}{c|}{PIQA}   & 7.19±.03     & 7.30±.02   & {\color[HTML]{FE0000} 1.53$\uparrow$}          & 7.43±.03        & {\color[HTML]{FE0000} 3.34$\uparrow$}    & 7.23±.02    & {\color[HTML]{FE0000} 0.56$\uparrow$}     \\

\multicolumn{1}{c|}{GSM8K}   & 10.50±.01  & 10.61±.05   & {\color[HTML]{FE0000} 1.05$\uparrow$}   & 10.69±.03    & {\color[HTML]{FE0000} 1.82$\uparrow$}     & 10.59±.03     & {\color[HTML]{FE0000} 0.86$\uparrow$}  \\ 

\bottomrule
\end{tabular}

\caption{\textit{Impact of Dual-Track Backpropagation, Bottom-Gradient Recovery, and Trunk Training on SL utility.}}
\label{tab:big_ablation}
\vspace{-5pt}
\end{table*}

\noindent\textbf{Impact of Bottom-Gradient Recovery.}
Disabling Bottom-Gradient Recovery likewise degrades PPL across all datasets, as the blinded gradient scale directly affects Bottom optimization. These results validate its role in preventing Scale Blinding from impairing training.

\noindent\textbf{Impact of Trunk Training.}
The gradients in the Trunk segment are subject to a certain degree of perturbation. Therefore, whether these gradients should be used to update the model is itself an important factor that warrants ablation. As shown in Tab.~\ref{tab:big_ablation}, we find that updating the trunk parameters does not compromise training stability; instead, it slightly improves training performance.

For the ablation studies of Dual-Track Backpropagation, Bottom-Gradient Recovery, and Trunk Training, the vMF mechanism uses a privacy budget of $\varepsilon=1024$, while SAS is configured with a total token ratio of $\rho=0.6$. The parameters of Scale Blinding are set to $m=1750$ and $r=1/7$.

\section{Conclusion}
We identify gradient--objective consistency as the root vulnerability enabling GMA-SL in LLM-SL, and propose Gradient Mirage, a defense that turns
gradient matching into a misspecified inverse problem. By injecting objective, directional, and scale inconsistencies into the exposed gradient while preserving full-label learning and recovering the effective bottom gradient, Gradient Mirage decouples what the model learns from what the gradient reveals. Experiments across multiple LLMs and datasets show that it substantially suppresses label reconstruction under comparable fine-tuning performance, achieving a superior privacy-utility trade-off. This work focuses on autoregressive training scenarios of large language models in split learning. Exploring gradient privacy attacks and defenses under other training paradigms and different model architectures, such as multimodal and vision models, remains a challenging direction.

\section*{Acknowledgments}

This work was partially funded by NSFC Grants No.62272222 and 62272215, Jiangsu Province Outstanding Youth Fund Project (No.BK20230080), the Nanjing University-China Mobile Communications Group Co.,Ltd. Joint Institute (NJ20250025).

\clearpage
\bibliography{aaai2026}

@inproceedings{SplitLLM,
  title={Unveiling the vulnerability of private fine-tuning in split-based frameworks for large language models: A bidirectionally enhanced attack},
  author={Chen, Guanzhong and Qin, Zhenghan and Yang, Mingxin and Zhou, Yajie and Fan, Tao and Du, Tianyu and Xu, Zenglin},
  booktitle={Proceedings of the 2024 on ACM SIGSAC Conference on Computer and Communications Security},
  pages={2904--2918},
  year={2024}
}

@inproceedings{unleashing,
  title={Unleashing the tiger: Inference attacks on split learning},
  author={Pasquini, Dario and Ateniese, Giuseppe and Bernaschi, Massimo},
  booktitle={Proceedings of the 2021 ACM SIGSAC conference on computer and communications security},
  pages={2113--2129},
  year={2021}
}

@inproceedings{DP-Forward,
  title={Dp-forward: Fine-tuning and inference on language models with differential privacy in forward pass},
  author={Du, Minxin and Yue, Xiang and Chow, Sherman SM and Wang, Tianhao and Huang, Chenyu and Sun, Huan},
  booktitle={Proceedings of the 2023 ACM SIGSAC Conference on Computer and Communications Security},
  pages={2665--2679},
  year={2023}
}

@article{SL1,
  title={Split learning for collaborative deep learning in healthcare},
  author={Poirot, Maarten G and Vepakomma, Praneeth and Chang, Ken and Kalpathy-Cramer, Jayashree and Gupta, Rajiv and Raskar, Ramesh},
  journal={arXiv preprint arXiv:1912.12115},
  year={2019}
}

@article{SL2,
  title={Distributed learning of deep neural network over multiple agents},
  author={Gupta, Otkrist and Raskar, Ramesh},
  journal={Journal of Network and Computer Applications},
  volume={116},
  pages={1--8},
  year={2018},
  publisher={Elsevier}
}

@article{SL3,
  title={Split learning for health: Distributed deep learning without sharing raw patient data},
  author={Vepakomma, Praneeth and Gupta, Otkrist and Swedish, Tristan and Raskar, Ramesh},
  journal={arXiv preprint arXiv:1812.00564},
  year={2018}
}

@inproceedings{SL4,
  title={SAP: Privacy-Preserving Fine-Tuning on Language Models with Split-and-Privatize Framework},
  author={Shen, Xicong and Liu, Yang and Liu, Yi and Wang, Peiran and Liu, Huiqi and Hong, Jue and Duan, Bing and Huang, Zirui and Mao, Yunlong and Wu, Ye and others},
  booktitle={Proceedings of the Thirty-Fourth International Joint Conference on Artificial Intelligence},
  pages={502--510},
  year={2025}
}

@InProceedings{SL5,
    author    = {Zhang, Qinbo and Shi, Yanhang and Zhang, Ziyi and Wang, Hao and Zhang, Sai Qian and Li, Jian},
    title     = {Stealing Split Learning Bottom Models by Recovering Embedding Geometry},
    booktitle = {Proceedings of the IEEE/CVF Conference on Computer Vision and Pattern Recognition (CVPR)},
    month     = {June},
    year      = {2026},
    pages     = {20660-20669}
}

@article{mlaas1,
  title={LLMaaS: Serving Large-Language Models on Trusted Serverless Computing Platforms},
  author={Cai, Zinuo and Ma, Rongbo and Fu, Yicheng and Zhang, Weishan and Ma, Ruhui and Guan, Haibing},
  journal={IEEE Transactions on Artificial Intelligence},
  volume={6},
  number={2},
  pages={405--415},
  year={2024},
  publisher={IEEE}
}

@article{mlaas2,
  title={Towards Privacy-Preserving LLM Inference via Covariant Obfuscation (Technical Report)},
  author={Lin, Yu and Zhang, Qizhi and Ruan, Wenqiang and Zhang, Daode and Hong, Jue and Wu, Ye and Xia, Hanning and Mao, Yunlong and Zhong, Sheng},
  journal={arXiv preprint arXiv:2603.01499},
  year={2026}
}

@inproceedings{dp,
  title={Calibrating noise to sensitivity in private data analysis},
  author={Dwork, Cynthia and McSherry, Frank and Nissim, Kobbi and Smith, Adam},
  booktitle={Theory of cryptography conference},
  pages={265--284},
  year={2006},
  organization={Springer}
}

@inproceedings{dp-sgd,
  title={Deep learning with differential privacy},
  author={Abadi, Martin and Chu, Andy and Goodfellow, Ian and McMahan, H Brendan and Mironov, Ilya and Talwar, Kunal and Zhang, Li},
  booktitle={Proceedings of the 2016 ACM SIGSAC conference on computer and communications security},
  pages={308--318},
  year={2016}
}

@article{dlg,
  title={Deep leakage from gradients},
  author={Zhu, Ligeng and Liu, Zhijian and Han, Song},
  journal={Advances in neural information processing systems},
  volume={32},
  year={2019}
}

@article{gradprune,
  title={Methods for pruning deep neural networks},
  author={Vadera, Sunil and Ameen, Salem},
  journal={Ieee Access},
  volume={10},
  pages={63280--63300},
  year={2022},
  publisher={IEEE}
}

@inproceedings{piqa,
  title={Piqa: Reasoning about physical commonsense in natural language},
  author={Bisk, Yonatan and Zellers, Rowan and Gao, Jianfeng and Choi, Yejin and others},
  booktitle={Proceedings of the AAAI conference on artificial intelligence},
  volume={34},
  number={05},
  pages={7432--7439},
  year={2020}
}

@misc{codealpaca,
  title={Code alpaca: An instruction-following llama model for code generation},
  author={Chaudhary, Sahil},
  year={2023}
}

@article{gsm8k,
  title={Training verifiers to solve math word problems, 2021},
  author={Cobbe, Karl and Kosaraju, Vineet and Bavarian, Mohammad and Chen, Mark and Jun, Heewoo and Kaiser, Lukasz and Plappert, Matthias and Tworek, Jerry and Hilton, Jacob and Nakano, Reiichiro and others},
  journal={URL https://arxiv. org/abs/2110.14168},
  volume={9},
  year={2021}
}

@inproceedings{aGM,
  title={Improving the gaussian mechanism for differential privacy: Analytical calibration and optimal denoising},
  author={Balle, Borja and Wang, Yu-Xiang},
  booktitle={International conference on machine learning},
  pages={394--403},
  year={2018},
  organization={PMLR}
}

@article{vonMises,
  title={{\"U}ber die “Ganzzahligkeit" der Atomgewichte und verwandete Fragen},
  author={von Mises, Richard},
  journal={Physikalische Zeitschrift},
  volume={19},
  pages={490},
  year={1918}
}

@article{fisher,
  title={Dispersion on a sphere},
  author={Fisher, Ronald Aylmer},
  journal={Proceedings of the royal society of London. Series A. Mathematical and physical sciences},
  volume={217},
  number={1130},
  pages={295--305},
  year={1953},
  publisher={The Royal Society London}
}

@inproceedings{DirDP,
  title={Differential privacy for directional data},
  author={Weggenmann, Benjamin and Kerschbaum, Florian},
  booktitle={Proceedings of the 2021 ACM SIGSAC Conference on Computer and Communications Security},
  pages={1205--1222},
  year={2021}
}

@article{gradient_drop,
  title={Safeguarding Federated Learning From Data Reconstruction Attacks via Gradient Dropout},
  author={Sotthiwat, Ekanut and Zhang, Chi and Xiao, Xiaokui and Zhen, Liangli},
  journal={TIFS},
  year={2026},
  publisher={IEEE}
}

@inproceedings{tag,
  title={Tag: Gradient attack on transformer-based language models},
  author={Deng, Jieren and Wang, Yijue and Li, Ji and Wang, Chenghong and Shang, Chao and Liu, Hang and Rajasekaran, Sanguthevar and Ding, Caiwen},
  booktitle={Findings of the association for computational linguistics: EMNLP 2021},
  pages={3600--3610},
  year={2021}
}

@inproceedings{CNNdaily,
  title={Get to the point: Summarization with pointer-generator networks},
  author={See, Abigail and Liu, Peter J and Manning, Christopher D},
  booktitle={Proceedings of the 55th Annual Meeting of the Association for Computational Linguistics (Volume 1: Long Papers)},
  pages={1073--1083},
  year={2017}
}

@inproceedings{rouge,
  title={Rouge: A package for automatic evaluation of summaries},
  author={Lin, Chin-Yew},
  booktitle={Text summarization branches out},
  pages={74--81},
  year={2004}
}

@inproceedings{meteor,
  title={METEOR: An automatic metric for MT evaluation with improved correlation with human judgments},
  author={Banerjee, Satanjeev and Lavie, Alon},
  booktitle={Proceedings of the acl workshop on intrinsic and extrinsic evaluation measures for machine translation and/or summarization},
  pages={65--72},
  year={2005}
}

@article{llama2,
  title={Llama 2: Open foundation and fine-tuned chat models},
  author={Touvron, Hugo and Martin, Louis and Stone, Kevin and Albert, Peter and Almahairi, Amjad and Babaei, Yasmine and Bashlykov, Nikolay and Batra, Soumya and Bhargava, Prajjwal and Bhosale, Shruti and others},
  journal={arXiv preprint arXiv:2307.09288},
  year={2023}
}

@article{llama3,
  title={The llama 3 herd of models},
  author={Grattafiori, Aaron and Dubey, Abhimanyu and Jauhri, Abhinav and Pandey, Abhinav and Kadian, Abhishek and Al-Dahle, Ahmad and Letman, Aiesha and Mathur, Akhil and Schelten, Alan and Vaughan, Alex and others},
  journal={arXiv preprint arXiv:2407.21783},
  year={2024}
}

@article{qwen3,
  title={Qwen3 technical report},
  author={Yang, An and Li, Anfeng and Yang, Baosong and Zhang, Beichen and Hui, Binyuan and Zheng, Bo and Yu, Bowen and Gao, Chang and Huang, Chengen and Lv, Chenxu and others},
  journal={arXiv preprint arXiv:2505.09388},
  year={2025}
}

@article{qwen2.5,
  title={Qwen2. 5 Technical Report},
  author={Yang, An and Yang, Baosong and Zhang, Beichen and Hui, Binyuan and Zheng, Bo and Yu, Bowen and Li, Chengyuan and Liu, Dayiheng and Huang, Fei and Wei, Haoran and others},
  journal={arXiv e-prints},
  pages={arXiv--2412},
  year={2024}
}

@article{deepseekllm,
  title={Deepseek llm: Scaling open-source language models with longtermism},
  author={Bi, Xiao and Chen, Deli and Chen, Guanting and Chen, Shanhuang and Dai, Damai and Deng, Chengqi and Ding, Honghui and Dong, Kai and Du, Qiushi and Fu, Zhe and others},
  journal={arXiv preprint arXiv:2401.02954},
  year={2024}
}

@article{idlg,
  title={idlg: Improved deep leakage from gradients},
  author={Zhao, Bo and Mopuri, Konda Reddy and Bilen, Hakan},
  journal={arXiv preprint arXiv:2001.02610},
  year={2020}
}

@article{lamp,
  title={Lamp: Extracting text from gradients with language model priors},
  author={Balunovic, Mislav and Dimitrov, Dimitar and Jovanovi{\'c}, Nikola and Vechev, Martin},
  journal={Advances in Neural Information Processing Systems},
  volume={35},
  pages={7641--7654},
  year={2022}
}

@article{film,
  title={Recovering private text in federated learning of language models},
  author={Gupta, Samyak and Huang, Yangsibo and Zhong, Zexuan and Gao, Tianyu and Li, Kai and Chen, Danqi},
  journal={Advances in neural information processing systems},
  volume={35},
  pages={8130--8143},
  year={2022}
}

@inproceedings{GRU,
  title={On the properties of neural machine translation: Encoder--decoder approaches},
  author={Cho, Kyunghyun and Van Merri{\"e}nboer, Bart and Bahdanau, Dzmitry and Bengio, Yoshua},
  booktitle={Proceedings of SSST-8, eighth workshop on syntax, semantics and structure in statistical translation},
  pages={103--111},
  year={2014}
}

\clearpage
\appendix


\begin{center}
    {\LARGE \textbf{SUPPLEMENTARY MATERIAL}}
\end{center}

The supplementary material provides additional details about our work. Specifically, we
provide:
\begin{itemize}
    \item More details of the vMF Mechanism.
    \item More details of Selective Autoregressive Supervision.
    \item Additional experimental details and analyses. 
    \item Visualization of the effectiveness of Gradient Mirage.
    \item Future Work.
\end{itemize}

\section{Details of the vMF Mechanism}
This section provides additional details on the vMF mechanism. Sections A.1 and A.2 describe our implementation of vMF sampling, including the tangent-normal decomposition and the rejection sampling procedure, respectively. Section A.3 provides the proof of the theoretical result stated in the main text.
\subsection{Tangent-Normal Decomposition in vMF Sampling}
The von Mises-Fisher privacy mechanism adopts the von Mises-Fisher (vMF) distribution, named after von Mises and Fisher. This distribution is defined over the unit hypersphere \(S^{n-1}\).
\begin{definition}[\textbf{The vMF Distribution}]
\label{def:vmf_dis}
The vMF distribution on $\mathbb{S}^{n-1}$ with mean direction $\mu \in \mathbb{S}^{n-1}$ and concentration parameter $\kappa \ge 0$ is defined by the following density:
\[
\mathrm{vMF}(\mu, \kappa)[\mathbf{x}]
=
C_{\mathrm{vMF}}(n,\kappa)
\cdot
\exp\left(\kappa \cdot \mu^{\top}\mathbf{x}\right),
\]
where $\mathbf{x} \in \mathbb{S}^{n-1}$ and $C_{\mathrm{vMF}}(n,\kappa)$ is the normalization constant.
\end{definition}

To sample a new direction vector  $\tilde{\mu} \in \mathbb{S}^{n-1}$ under the vMF mechanism centered at $\mu \in \mathbb{S}^{n-1}$, we use the tangent-normal decomposition. As illustrated in Fig.~\ref{fig:Appendix_VMF}, the sampled vector \( \tilde{\mu} \) is decomposed into a component along \( \mu \) and a component in the subspace orthogonal to \( \mu \). Let
$t=\tilde{\mu}^{\top}\mu=\cos\theta ,$
where \( \theta \) is the angle between \( \tilde{\mu} \) and \( \mu \). Then the orthogonal component has magnitude
$h=\sqrt{1-t^2}=\sin\theta$.
If \( \xi \sim \mathrm{Unif}(\mathbb{S}^{n-2}\perp \mu) \) is a unit vector uniformly sampled from the subsphere orthogonal to \( \mu \), the sampled direction can be written as
$\tilde{\mu}=t\mu+\sqrt{1-t^2}\,\xi$
or equivalently
$\tilde{\mu}=\cos\theta\,\mu+\sin\theta\,\xi$.
Therefore, vMF sampling is separated into two parts: sampling the axial projection \( t \) (or equivalently the angle \( \theta \)) and sampling the tangential direction \( \xi \) uniformly on the orthogonal subsphere.

\subsection{Rejection Sampling for $t$}
After the tangent-normal decomposition, vMF sampling reduces to drawing the axial variable \(t \in [-1,1]\). In our implementation, \(t\) is generated by a rejection sampling scheme with a Beta proposal. Specifically, we first sample 
\begin{equation}
    y \sim \mathrm{Beta}\!\left(\frac{p-1}{2}, \frac{p-1}{2}\right),
\end{equation}
and transform it into a candidate $t=\frac{1-(1+b_0)y}{1-(1-b_0)y},$
where $b_0=\frac{-2\kappa+\sqrt{4\kappa^2+(p-1)^2}}{p-1}.$
Following the standard construction, we further define
\begin{equation}
a=\frac{p-1+2\kappa+\sqrt{4\kappa^2+(p-1)^2}}{4},
\end{equation}
and 
\begin{equation}
d=\frac{4ab_0}{1+b_0}-(p-1)\log(p-1).
\end{equation}
For each proposed sample, let
\begin{equation}
s=\frac{2ab_0}{1-(1-b_0)y}.
\end{equation}
The candidate is accepted if
\begin{equation}
\log u \le (p-1)\log s - s + d,
\end{equation}
where \(u \sim \mathrm{Unif}(0,1)\). Otherwise, the sample is rejected and the procedure is repeated. This yields a valid sample for \(t\), which is then combined with the tangential direction to form the final vMF sample.

\begin{figure}[t]
    \centering
    \includegraphics[width=0.95\linewidth]{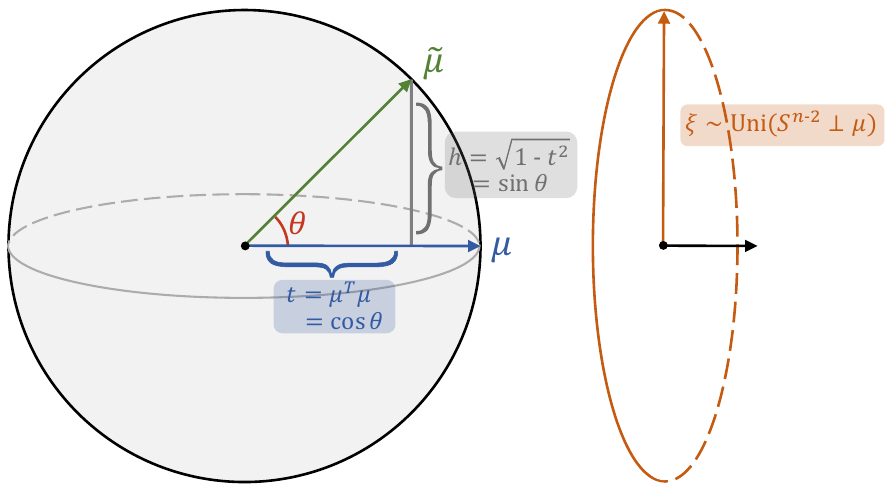}
    \caption{\textit{Tangent-normal decomposition in vMF sampling.} The sampled direction \( \tilde{\mu} \) is decomposed into an axial component \( t=\tilde{\mu}^{\top}\mu=\cos\theta \) along \( \mu \) and a tangential component of magnitude \( \sqrt{1-t^2}=\sin\theta \), where \( \xi \sim \mathrm{Unif}(\mathbb{S}^{n-2}\perp\mu) \).}
    \label{fig:Appendix_VMF}
\end{figure}

\subsection{Proofs related to the vMF Mechanism}

\begin{theorem}[\textbf{The vMF Mechanism Satisfies $\varepsilon d_{2}$-Privacy}]
\label{theo:satisfy_L2}
Let \(\mathcal{M}_{\mathrm{vMF}}\) denote the vMF mechanism. Then $\mathcal{M}_{\mathrm{vMF}}$ satisfies \(\varepsilon d_2\)-metric differential privacy, where
\[
d_2(\mathbf{u},\mathbf{v})=\|\mathbf{u}-\mathbf{v}\|_2.\]
For any \(\mathbf{u},\mathbf{v}\in\mathbb{S}^{d-1}\) and any measurable set \(S\subseteq\mathbb{S}^{d-1}\),
\[\Pr[\mathcal{M}_{\mathrm{vMF}}(\mathbf{u})\in S]
\le
\exp\!\big(\varepsilon\, d_2(\mathbf{u},\mathbf{v})\big)\,
\Pr[\mathcal{M}_{\mathrm{vMF}}(\mathbf{v})\in S].\]
\end{theorem}

\begin{proof}[Proof of Theorem~\ref{theo:satisfy_L2}]
 Let \(\mathbf{u},\mathbf{v}\in\mathbb{S}^{d-1}\) be arbitrary unit vectors, and let \(S\subseteq\mathbb{S}^{d-1}\) be any measurable set. For any \(\mathbf{z}\in S\), the vMF density satisfies
\[
\frac{p_{\mathbf{u}}(\mathbf{z})}{p_{\mathbf{v}}(\mathbf{z})}
=
\frac{C_d(\varepsilon)\exp(\epsilon\,\mathbf{u}^{\top}\mathbf{z})}
{C_d(\epsilon)\exp(\epsilon\,\mathbf{v}^{\top}\mathbf{z})}
=\exp\!\big(\epsilon(\mathbf{u}-\mathbf{v})^{\top}\mathbf{z}\big).
\]
By the Cauchy--Schwarz inequality,
\[
(\mathbf{u}-\mathbf{v})^{\top}\mathbf{z}
\le
\|\mathbf{u}-\mathbf{v}\|_2\,\|\mathbf{z}\|_2.
\]
Since \(\mathbf{z}\in\mathbb{S}^{d-1}\), we have \(\|\mathbf{z}\|_2=1\). Hence,
\[
\frac{p_{\mathbf{u}}(\mathbf{z})}{p_{\mathbf{v}}(\mathbf{z})}
\le
\exp\!\big(\varepsilon\|\mathbf{u}-\mathbf{v}\|_2\big)
=
\exp\!\big(\varepsilon d_2(\mathbf{u},\mathbf{v})\big).
\]
Integrating both sides over \(\mathbf{z}\in S\), we obtain
\begin{align*}
\Pr[\mathcal{M}_{\mathrm{vMF}}(\mathbf{u})\in S]
=&
\int_S p_{\mathbf{u}}(\mathbf{z})\,d\mathbf{z} \\
\le&
\exp\!\big(\varepsilon d_2(\mathbf{u},\mathbf{v})\big)
\int_S p_{\mathbf{v}}(\mathbf{z})\,d\mathbf{z} \\
=&
\exp\!\big(\varepsilon d_2(\mathbf{u},\mathbf{v})\big)
\Pr[\mathcal{M}_{\mathrm{vMF}}(\mathbf{v})\in S].
\end{align*}
Therefore, \(\mathcal{M}_{\mathrm{vMF}}\) satisfies \(\varepsilon d_2\)-metric differential privacy. 
\end{proof}

\begin{theorem}[\textbf{The vMF Mechanism Satisfies $\varepsilon d_{\angle}$-Privacy}]
\label{theo:satisfy_angle}
\(\mathcal{M}_{\mathrm{vMF}}\) also satisfies \(\varepsilon d_{\angle}\)-metric differential privacy, where
\[d_{\angle}(\mathbf{u},\mathbf{v})=\arccos(\mathbf{u}^{\top}\mathbf{v}).\]
For any \(\mathbf{u},\mathbf{v}\in\mathbb{S}^{d-1}\) and any measurable set \(S\subseteq\mathbb{S}^{d-1}\),
\[\Pr[\mathcal{M}_{\mathrm{vMF}}(\mathbf{u})\in S]
\leq
\exp\left(\varepsilon d_{\angle}(\mathbf{u},\mathbf{v})\right)
\Pr[\mathcal{M}_{\mathrm{vMF}}(\mathbf{v})\in S].\]
\end{theorem}

\begin{proof}[Proof of Theorem~\ref{theo:satisfy_angle}]
From Theorem~\ref{theo:satisfy_L2}, the vMF mechanism satisfies
\[\Pr[\mathcal{M}_{\mathrm{vMF}}(\mathbf{u})\in S] \leq \exp\left(\varepsilon\|\mathbf{u}-\mathbf{v}\|_2\right)
\Pr[\mathcal{M}_{\mathrm{vMF}}(\mathbf{v})\in S].\]
It remains to show that
$\|\mathbf{u}-\mathbf{v}\|_2 \leq d_{\angle}(\mathbf{u},\mathbf{v}).$
Let \(\theta=d_{\angle}(\mathbf{u},\mathbf{v})\in[0,\pi]\). Since \(\mathbf{u}\) and \(\mathbf{v}\) are unit vectors,
\[\|\mathbf{u}-\mathbf{v}\|_2 = \sqrt{2-2\mathbf{u}^{\top}\mathbf{v}} =
\sqrt{2-2\cos\theta} = 2\sin\frac{\theta}{2}.\]
For any \(\theta\in[0,\pi]\), we have
$2\sin\frac{\theta}{2} \leq \theta.$
Therefore,
$\|\mathbf{u}-\mathbf{v}\|_2 \leq d_{\angle}(\mathbf{u},\mathbf{v}).$
Substituting this inequality into the result of Theorem 2 yields
\[\exp\left(\varepsilon\|\mathbf{u}-\mathbf{v}\|_2\right)
\leq\exp\left(\varepsilon d_{\angle}(\mathbf{u},\mathbf{v})\right).\]
Hence,
\[
\begin{aligned}
\Pr[\mathcal{M}_{\mathrm{vMF}}(\mathbf{u})\in S]
&\leq
\exp\!\left(\varepsilon d_{\angle}(\mathbf{u},\mathbf{v})\right) \\
&\quad \cdot
\Pr[\mathcal{M}_{\mathrm{vMF}}(\mathbf{v})\in S].
\end{aligned}
\]
Thus, the mechanism satisfies
$\varepsilon d_{\angle}$-metric differential privacy.
\end{proof}

\begin{proof}[Proof of Theorem~\ref{theo:2}]
Since
$
\mathbf{g}
=
\|\mathbf{g}\|_2\mathbf{u}
\quad
\text{and}
\quad
\tilde{\mathbf{g}}
=
\|\mathbf{g}\|_2\tilde{\mathbf{u}},
$
the perturbation noise can be written as
\[
\mathbf{n}
=
\tilde{\mathbf{g}}-\mathbf{g}
=
\|\mathbf{g}\|_2(\tilde{\mathbf{u}}-\mathbf{u}).
\]
Therefore,
$\|\mathbf{n}\|_2=\|\mathbf{g}\|_2\|\tilde{\mathbf{u}}-\mathbf{u}\|_2.$
Because both \(\mathbf{u}\) and \(\tilde{\mathbf{u}}\) are unit vectors, their Euclidean distance is at most \(2\):
\[\|\tilde{\mathbf{u}}-\mathbf{u}\|_2
\leq\|\tilde{\mathbf{u}}\|_2+\|\mathbf{u}\|_2=2.\]
Thus,
$
\|\mathbf{n}\|_2
\leq
2\|\mathbf{g}\|_2.
$
It follows that
\[\mathrm{ASNR}=\frac{\|\mathbf{g}\|_2}{\|\mathbf{n}\|_2}
\geq\frac{\|\mathbf{g}\|_2}{2\|\mathbf{g}\|_2}=\frac{1}{2}.\]
The equality holds when
$\|\tilde{\mathbf{u}}-\mathbf{u}\|_2=2,$
which occurs if and only if \(\tilde{\mathbf{u}}=-\mathbf{u}\). Hence, the bound is tight.
\end{proof}

\section{Details of Selective Autoregressive Supervision}
\subsection{Last-$k$ Tokens in Selective Autoregressive Supervision}

\begin{figure}[t]
    \centering
    \includegraphics[width=0.85\linewidth]{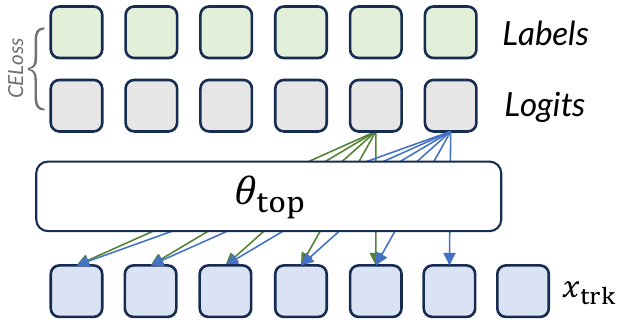}
    \caption{\textit{ Illustration of the zero-gradient suffix induced by masking the last \(k\) target tokens in selective autoregressive supervision.} The arrows indicate the direction of gradient propagation.}
    \label{fig:lastk}
    \vspace{-8pt}
\end{figure}

\begin{figure*}[t]
    \centering
    \includegraphics[width=1\linewidth]{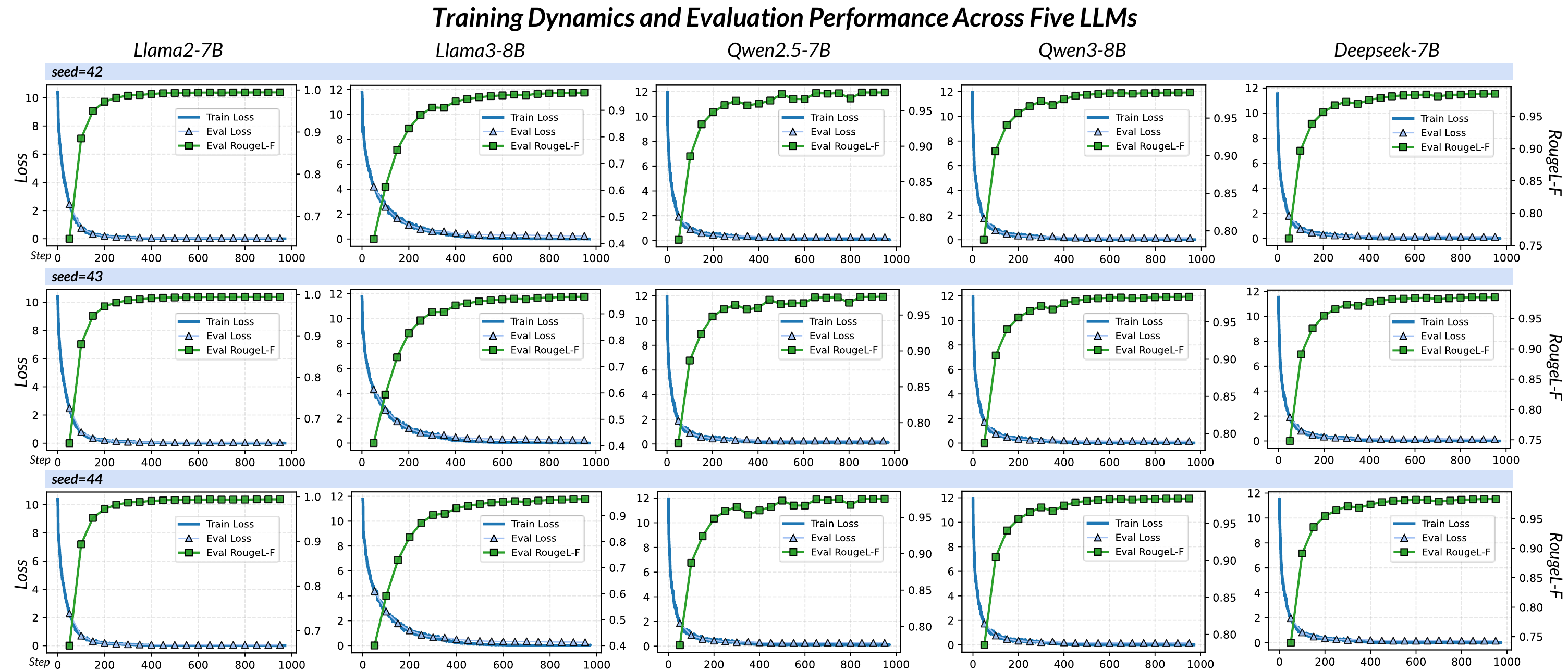}
    \caption{\textit{Training dynamics of the learning-based inversion decoder across five LLMs.} Shown are the training loss, evaluation loss, and evaluation RougeL-F over training steps under three random seeds.}
    \label{fig_Appendix:LiD}
\end{figure*}

We next discuss a practical implementation detail of SAS from the client's perspective. 

While Selective Autoregressive Supervision is designed to perturb the supervision structure exposed through the split-interface gradient, its implementation should also preserve the stealthiness of the defense. 

In particular, if the last \(k\) consecutive target tokens are all excluded from the loss computation, then the corresponding gradient signals on a contiguous block of \(k\) token representations in \(x_{\mathrm{trk}}\) become identically zero. 
Such a structured zero-gradient suffix can serve as an explicit artifact of the defense and may reveal the existence of selective masking to an adversarial server.

To see this more formally, consider the autoregressive loss with a binary supervision mask \(m_t \in \{0,1\}\):
\begin{equation}
\mathcal{L}_{M}=
\frac{1}{\sum_{t=1}^{T} m_t}
\sum_{t=1}^{T} m_t \ell_t,
\
\ell_t = -\log p_t(y_t),
\end{equation}
where \(p_t\) denotes the predictive distribution for the target token \(y_t\). 
Let \(h_i\) denote the split-interface representation at position \(i\), i.e., the \(i\)-th token representation in \(x_{\mathrm{trk}}\). 
Due to the causal structure of autoregressive decoding, the representation \(h_i\) can only influence the losses of future or aligned target positions. 
Abstractly, we may write
\begin{equation}
\frac{\partial \mathcal{L}_{M}}{\partial h_i}=\frac{1}{\sum_{t=1}^{T} m_t}
\sum_{t \in \mathcal{D}(i)}
m_t
\frac{\partial \ell_t}{\partial h_i},
\end{equation}
where \(\mathcal{D}(i)\) denotes the set of supervised prediction positions whose computation depends on \(h_i\). 
For the last few token positions, this dependency set contains only a small number of subsequent losses. 
Therefore, if the last \(k\) target positions are all masked out, namely $m_{T-k+1}=m_{T-k+2}=\cdots=m_T=0,$
then for the corresponding suffix positions whose only downstream supervised losses lie in this masked region, we obtain
$\frac{\partial \mathcal{L}_{M}}{\partial h_i}=0.$
Consequently, the gradient returned to the server-side Trunk segment contains a contiguous zero-gradient block over the suffix of \(x_{\mathrm{trk}}\). 
This phenomenon is illustrated in Fig.~\ref{fig:lastk}.

This artifact is undesirable because it weakens the concealment of the defense. 
From the client's perspective, the goal of SAS is not only to introduce a supervision mismatch against gradient-matching attacks, but also to avoid leaving obvious structural signatures in the exposed gradient. 
A consecutive zero-gradient suffix is particularly problematic because it is independent of token content and can be detected directly from the norm pattern of the interface gradient. 
An adversary may observe that
$\left\|
\frac{\partial \mathcal{L}_{M}}{\partial h_i}
\right\|_2 = 0$
for several consecutive positions near the end of the sequence, and infer that these positions were systematically removed from supervision.

To avoid this issue, SAS should ensure that the last \(k\) token positions are not simultaneously excluded from the masked loss. 
Equivalently, the supervision mask should satisfy
$\sum_{t=T-k+1}^{T} m_t \geq 1.$
This constraint guarantees that at least one suffix target token contributes to the loss, thereby preventing a fully zero-gradient suffix from appearing in \(x_{\mathrm{trk}}\). 
To prevent such a zero-gradient suffix, we enforce the last \(k\) tokens to be valid supervised positions in our implementation. We set \(k=5\) throughout our experiments.

\section{More Experimental Details}
\label{sec:Appendix_C}

\subsection{Definitions of Gradient Utility Metrics}
\label{sec:Appendix_C1}
To quantify the utility of privatized gradients, we use four complementary metrics: the Amplitude Signal-to-Noise Ratio (ASNR), ASNR@10\%, Recall@10\% (R@10\%), and Jaccard@10\% (J@10\%). Let \(g \in \mathbb{R}^{LD}\) denote the original gradient and \(\tilde{g} \in \mathbb{R}^{LD}\) the privatized gradient. The injected noise is defined as $n=\tilde{g}-g.$

\noindent\textbf{ASNR.} The Amplitude Signal-to-Noise Ratio measures the overall gradient magnitude relative to the perturbation magnitude:
\[\mathrm{ASNR}(g,\tilde{g})=\frac{\|g\|_2}{\|n\|_2}.\]

\noindent\textbf{ASNR@10\%.} Let \(S_{10}(g)\) denote the index set of the top 10\% entries of \(g\) ranked by absolute value. ASNR@10\% is computed on this subset:
\[\mathrm{ASNR@10\%}(g,\tilde{g})=
\frac{\|g_{S_{10}(g)}\|_2}{\|n_{S_{10}(g)}\|_2}.\]

\noindent\textbf{Recall@10\%.} Let \(S_{10}(g)\) and \(S_{10}(\tilde{g})\) be the top 10\% index sets of \(g\) and \(\tilde{g}\), respectively. Recall@10\% measures the fraction of salient coordinates in the original gradient that are retained after privatization:
\[\mathrm{R@10\%}(g,\tilde{g})=
\frac{|S_{10}(g)\cap S_{10}(\tilde{g})|}{|S_{10}(g)|}.\]

\noindent\textbf{Jaccard@10\%.} Jaccard@10\% measures the overlap between the two top-10\% index sets:
\[\mathrm{J@10\%}(g,\tilde{g})=
\frac{|S_{10}(g)\cap S_{10}(\tilde{g})|}{|S_{10}(g)\cup S_{10}(\tilde{g})|}.\]
ASNR and ASNR@10\% evaluate magnitude preservation, while R@10\% and J@10\% measure support preservation.

\begin{figure*}[t]
    \centering
    \includegraphics[width=0.95\linewidth]{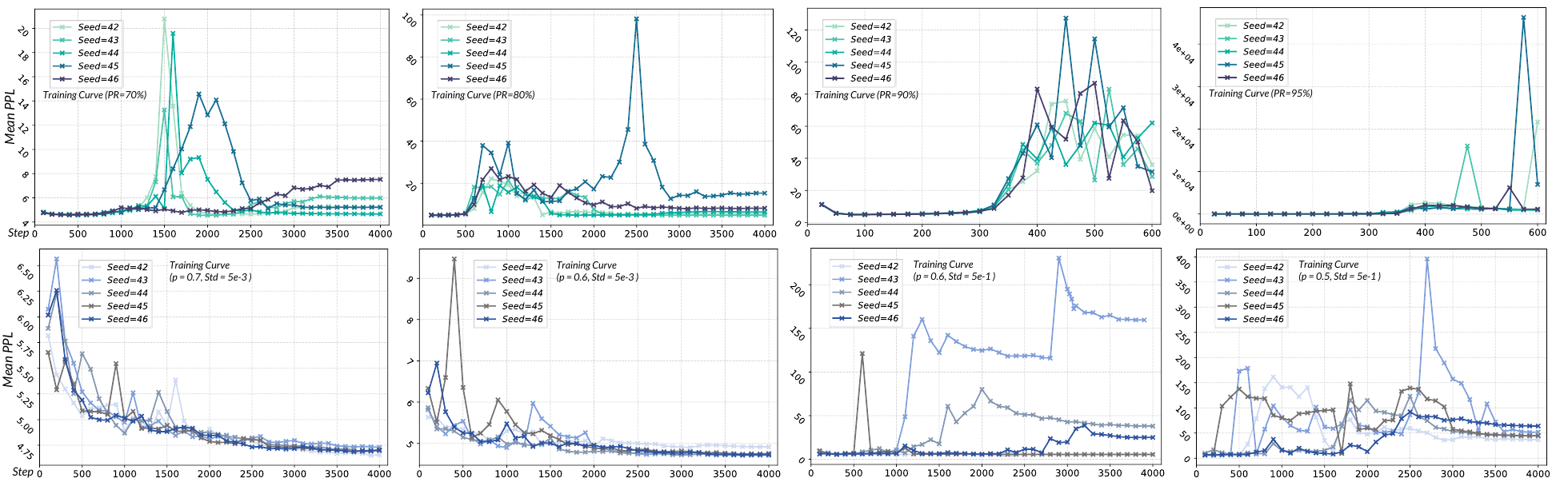}
    \caption{\textit{Training Curves of Gradient Pruning and Gradient Dropout.}}
    \label{fig_Appendix:GP}
\end{figure*}
\begin{table*}[t]
\centering

\renewcommand{\arraystretch}{1.05}
\resizebox{0.95\textwidth}{!}{%
\begin{tabular}{cccccccccc}
\toprule
Dataset   & Method   &  Final PPL & PPL@500steps & PPL@1000steps  &  RougeL-F      & Rouge1-F           & Rouge2-F           & Meteor             & TRR                \\ 
\midrule
\multirow{8}{*}{\begin{tabular}[c]{@{}c@{}}Code-\\ Alpaca\end{tabular}} & Grad Pruning       & 5.292{\scriptsize $\pm$.538}   & 4.595±.001   & 4.934±.017    & .995±.001          & .995±.001          & .981±.001          & .987±.001          & .990±.000          \\
                 & GradSeq-LDP(10240) & 4.670±.002  & 4.962±.055  & 4.782±.043  & .564±.032    & .567±.037   & .413±.020          & .570±.016     & .634±.007       \\
                & GradSeq-LDP(12288) & 4.611±.037   & 4.915±.022   & 4.804±.050 & .657±.029     & .661±.030     & .517±.042       & .664±.042       & .717±.041     \\
                & Grad Dropout       & 4.782±.104   & 5.341±.158   & 4.954±.041    & .942±.008     & .942±.008     & .915±.012   & .940±.007          & .943±.013          \\
                                        & \textbf{Gradient Mirage(512)}        & 4.451±.001   & 4.671±.022   & 4.610±.008   & \textbf{.202±.002} & \textbf{.208±.003} & \textbf{.108±.001} & \textbf{.226±.000} & \textbf{.255±.002} \\
                                        & \textbf{Gradient Mirage(1024)}         & 4.421±.008   & 4.603±.008   & 4.552±.013   & \textbf{.351±.022} & \textbf{.361±.022} & \textbf{.177±.020} & \textbf{.341±.017} & \textbf{.371±.009} \\
                            
                                        & Top-Only Training  & 4.831±.005   & 5.190±0.006  & 4.970±.004    & -       & -    & -         & -          & -          \\
                                        & Standard Training  & 4.444±.004   & 4.504±.012   & 4.447±.005    & 1.000±.000    & 1.000±.000         & .989±.002   & .991±.002    & .992±.001    \\ 
\midrule
\multirow{8}{*}{PIQA}                    & Grad Pruning       & 11.500±2.161 & 9.229±.598   & 13.624±.095   & .959±.011          & .959±.011          & .908±.023          & .935±.017        & .938±.013          \\
                                        & GradSeq-LDP(10240) & 7.738±.056   & 7.443±.025  & 7.212±.017  & .799±.034  & .804±.033   & .613±.039    & .760±.048    & .758±.028     \\
                                        & GradSeq-LDP(12288) & 7.202±.019   & 7.726±.040   & 7.416±.021    & .805±.033 & .810±.032          & .626±.056    & .769±.049   & .762±.039          \\
                                        & Grad Dropout       & 7.758±.029   & 9.111±.265   & 8.316±.145    & .950±.006          & .950±.006   & .867±.002    & .912±.003    & .906±.008    \\
                                    & \textbf{Gradient Mirage(512)}   & 7.404±.020   & 8.120±.042   & 7.894±.100    & \textbf{.209±.017} & \textbf{.215±.017} & \textbf{.040±.008} & \textbf{.137±.013} & \textbf{.183±.013} \\
                                    & \textbf{Gradient Mirage(1024)}   & 7.191±.033   & 7.824±.043   & 7.456±.035    & \textbf{.448±.020} & \textbf{.455±.019} & \textbf{.145±.014} & \textbf{.306±.025} & \textbf{.374±.007} \\
                                        & Top-Only Training  & 8.874±.018   & 9.613±.027   & 9.269±.029    & -                  & -               & -                  & -                  & -                  \\
                                           & Standard Training  & 7.396±.013   & 7.338±.036   & 7.106±.017    & .993±.004          & .993±.004          & .969±.010          & .975±.007          & .978±.005          \\ 
\midrule
\multirow{8}{*}{GSM8K}                                                  & Grad Pruning       & 12.227±1.131 & 10.769±.059  & 10.973±.301   & .997±.003          & .997±.006          & .994±.003          & .996±.012          & .996±.002          \\
                                            & GradSeq-LDP(10240) & 11.038±.025  & 11.712±.152  & 11.431±.005   & .904±.014          & .904±.014          & .843±.013          & .923±.007          & .914±.007          \\
                                          & GradSeq-LDP(12288) & 11.122±.039  & 11.669±.114  & 11.423±.058   & .885±.019          & .886±.018          & .821±.024          & .910±.013          & .901±.013          \\
                                        & Grad Dropout       & 12.477±.972  & 12.556±.153  & 12.365±.172   & .949±.010          & .949±.011          & .904±.015          & .940±.010          & .927±.011          \\
                                          & \textbf{Gradient Mirage(512)}     & 10.643±.031  & 10.958±.051  & 10.806±.055   & \textbf{.112±.008} & \textbf{.119±.007} & \textbf{.011±.005} & \textbf{.070±.006} & \textbf{.079±.008} \\
                                           & \textbf{Gradient Mirage(1024)}     & 10.500±.011  & 10.710±.042  & 10.587±.034   & \textbf{.363±.012} & \textbf{.387±.009} & \textbf{.087±.007} & \textbf{.213±.013} & \textbf{.243±.017} \\
                                          & Top-Only Training  & 11.999±.007  & 12.283±.005  & 12.075±.004   & -      & -      & -       & -          & -                  \\
                                            & Standard Training  & 10.588±.003  & 10.471±.059  & 10.416±.016   & 1.000±.000         & 1.000±.000         & .993±.001          & .995±.001          & .996±.000          \\ \bottomrule
\end{tabular}

}

\caption{\textit{Comparison between Gradient Mirage and existing
privacy-preserving methods with Llama3-8B.}}
\label{tab_appendix:big_llama3}
\vspace{-3pt}
\end{table*}

\subsection{Training Details of the Learning-based Inversion Decoder}
Here we present the training details of the learning-based inversion decoder used in the semi-white-box Forward Inversion Paradigm (SIP)~\cite{SplitLLM}, which corresponds to the “SIP-only” setting in our experiments. The goal of this decoder is to directly reconstruct the original input token sequence from the smashed data produced by the client-side Bottom segment.

Following the setting of the original work, we use the open-source CNN-DailyMail News Text Summarization dataset as the auxiliary training data~\cite{CNNdaily}. Specifically, the training data is derived from the CNN-DailyMail news summarization corpus.

The inversion decoder takes the smashed data generated by the Bottom segment as input. In our setting, the Bottom segment contains 6 decoder layers. Given the hidden representations with embedding dimension $D$, the inverter uses a single-layer unidirectional GRU~\cite{GRU} with hidden size 4096 to model the sequential dependency of the intermediate activations. The GRU output at each token position is then passed through a linear projection layer, which maps the hidden state to the vocabulary space and produces token-level logits. A dropout rate of 0.1 is applied after the GRU layer during training.

The inverter is trained in an auto-encoder-style manner. The Bottom segment of the target language model is used as a frozen encoder to generate intermediate representations from the auxiliary text data, while only the inversion decoder is optimized. The training objective is token-level cross-entropy loss between the decoder output and the original input token sequence. In this way, the decoder learns to map the smashed data back to the corresponding textual tokens.

\begin{figure*}[t]
    \centering
    \includegraphics[width=0.98\linewidth]{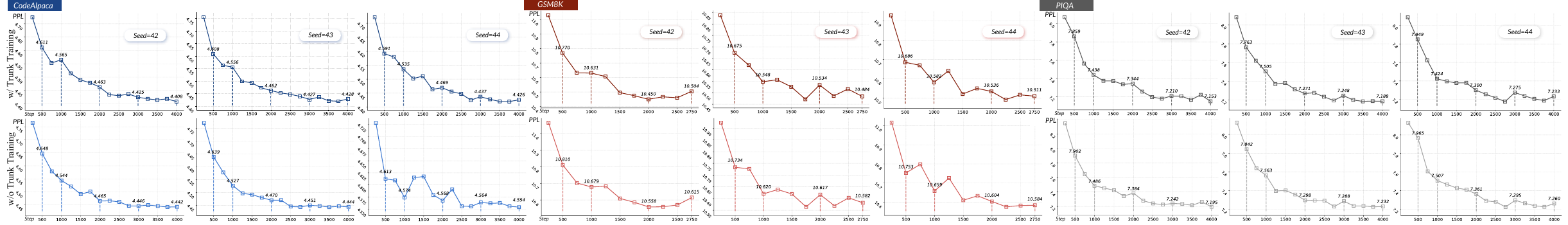}
    \caption{\textit{Comparison between the training dynamics with and without Trunk segment updates on Llama3-8B across three datasets.} Updating the Trunk segment generally leads to lower validation perplexity across random seeds.}
    \label{fig_Appendix:Trunk_curve}
\end{figure*}
\begin{table*}[t]
\centering
\renewcommand{\arraystretch}{1.05}
\resizebox{0.95\textwidth}{!}{%

\begin{tabular}{cccccccccc}
\toprule
Dataset   & Method    & Final PPL    & PPL@500steps & PPL@1000steps & RougeL-F  & Rouge1-F  & Rouge2-F  & Meteor    & TRR       \\ 
\midrule
\multirow{8}{*}{\begin{tabular}[c]{@{}c@{}}Code-\\ Alpaca\end{tabular}} & Grad Pruning   & 5.000±.089    & 4.864±.088    & 4.765±.044      & .993±.003 & .993±.003 & .988±.006 & .989±.005 & .995±.003 \\
                            & GradSeq-LDP(10240) & 4.804±.083      & 5.141±.149                 & 4.991±.155    & .775±.033 & .776±.034  & .656±.048 & .741±.043 & .886±.018 \\
                            & GradSeq-LDP(12288) & 4.853±.131                          & 5.132±.143                                              & 5.024±.172                                               & .820±.037 & .821±.037 & .725±.052 & .795±.043 & .909±.018 \\
                            & Grad Dropout   & 6.979±1.575   & 7.329±.929    & 82.223±129.115      & .902±.017 & .902±.016 & .859±.026 & .890±.016 & .932±.008 \\
                            & \textbf{Gradient Mirage(512)}     & 4.817±.098    & 5.186±.124    & 5.066±.068   & \textbf{.095±.012} & \textbf{.097±.012} & \textbf{.016±.002} & \textbf{.112±.005} & \textbf{.199±.010} \\
                            & \textbf{Gradient Mirage(1024)}     & 4.772±.049    & 5.031±.136    & 5.061±.124   & \textbf{.174±.020} & \textbf{.177±.019} & \textbf{.042±.007} & \textbf{.146±.008} & \textbf{.297±.004} \\
                            & Top-Only Training  & 4.961±.040    & 5.231±.158   & 5.025±.083          & -         & -         & -         & -         & -         \\
                            & Standard Training  & 4.731±.018      & 4.776±.108   & 4.669±.035        & .990±.003 & .990±.003 & .981±.006 & .982±.005 & .992±.003 \\ 
\midrule
\multirow{8}{*}{PIQA}       & Grad Pruning       & 6.501±.615    & 5.849±.055                                              & 5.576±.188                                               & .996±.004 & .996±.004 & .991±.009 & .995±.005 & .997±.003 \\
                            & GradSeq-LDP(10240) & 6.302±.656                          & 6.325±.248                                              & 6.280±.208                                               & .792±.038 & .793±.038 & .635±.057 & .744±.043 & .861±.008 \\
                            & GradSeq-LDP(12288) & 5.851±.049                          & 6.147±.067                                              & 6.172±.256                                               & .826±.017 & .826±.017 & .691±.026 & .786±.023 & .883±.005 \\
                            & Grad Dropout       & 7.757±1.772                         & 11.069±.524                                             & 8.452±1.022                                              & .813±.020 & .815±.017 & .743±.026 & .798±.019 & .859±.008 \\
                            & \textbf{Gradient Mirage(512)} & 5.631±.130 & 5.770±.095 & 5.767±.052 & \textbf{.047±.015} & \textbf{.047±.016} & \textbf{.001±.002} & \textbf{.046±.001} & \textbf{.115±.007} \\
                            
                            & \textbf{Gradient Mirage(1024)}   & 5.561±.232   & 5.748±.061      & 5.607±.065    & \textbf{.102±.012} & \textbf{.103±.011} & \textbf{.016±.003} & \textbf{.088±.013} & \textbf{.203±.016}  \\
                            & Top-Only Training  & 6.272±.059                          & 6.661±.129                                              & 6.465±.208                                               & -         & -         & -         & -         & -         \\
                            & Standard Training  & 5.285±.095                          & 5.380±.113                                              & 5.339±.060                                               & .994±.006 & .994±.006 & .988±.012 & .989±.011 & .994±.006 \\ 
\midrule
\multirow{8}{*}{GSM8K}      & Grad Pruning       & 10.659±.230                         & 10.520±.089           & 10.498±.164                                              & .978±.005 & .978±.005 & .970±.005 & .982±.002 & .991±.001 \\
                            & GradSeq-LDP(10240) & 10.487±.152                         & 10.621±.052                                             & 10.337±.134                                              & .375±.021 & .390±.019 & .203±.019 & .333±.013 & .629±.014 \\
                            & GradSeq-LDP(12288) & 10.483±.143                         & 10.608±.062                                             & 10.338±.134                                              & .440±.028 & .456±.027 & .264±.028 & .404±.029 & .680±.016 \\
                            & Grad Dropout       & \multicolumn{1}{l}{107.082±162.999} & \multicolumn{1}{l}{48.571±61.245}                       & \multicolumn{1}{l}{73.177±103.329}                       & .897±.011 & .897±.011 & .880±.004 & .921±.003 & .947±.005 \\
   
                            & \textbf{Gradient Mirage(512)} & 10.430±.159 & 10.760±.207 & 10.302±.055 & \textbf{.031±.009} & \textbf{.031±.009} & \textbf{.002±.001} & \textbf{.042±.006} & \textbf{.082±.012} \\
                            & \textbf{Gradient Mirage(1024)}    & 10.142±.169      & 10.387±.257     & 10.237±.300     & \textbf{.075±.012}  & \textbf{.078±.012} & \textbf{.005±.002} & \textbf{.071±.011} & \textbf{.154±.024}  \\
                            & Top-Only Training  & 11.198±.279                         & 11.644±.090                                             & 11.295±.172                                              & -         & -         & -         & -         & -         \\
                            & Standard Training  & 9.714±.138                          & 9.934±.289                                              & 9.786±.168                                               & .977±.001 & .977±.001 & .968±.001 & .981±.000 & .991±.000 \\ \bottomrule
\end{tabular}

}

\caption{\textit{Comparison between Gradient Mirage and existing privacy-preserving methods with Llama2-7B.}}
\label{tab_appendix:big_llama2}

\end{table*}

We train the inversion decoder for 3 epochs with a batch size of 32. The model is optimized using AdamW with a learning rate of 1e-3 and a weight decay of 1e-5. After training, the learned inverter can be directly applied to the smashed data observed during split fine-tuning to produce an initial reconstruction of the client’s private input sequence. Fig.~\ref{fig_Appendix:LiD} illustrates the training dynamics of the decoder on five LLMs~\cite{llama2,llama3, qwen2.5, qwen3, deepseekllm}.

\subsection{Training Curves of GP and GD}
Fig.~\ref{fig_Appendix:GP} illustrates the training curves of GP and GD. It can be observed that, to achieve strong privacy protection with GP or GD, the training process becomes extremely unstable, making it challenging to achieve a favorable privacy-utility trade-off.

\subsection{Ablation Details of Trunk Training}
Fig.~\ref{fig_Appendix:Trunk_curve} illustrates the differences in training dynamics with and without Trunk segment updates, evaluated on Llama3-8B across three datasets. For each dataset, the upper panel corresponds to training with Trunk segment updates, which consistently yields better performance.

\subsection{Extended Defense Experiments}
\label{sec:c5}
In the main text, we evaluate defense performance using only ROUGE-L-F and METEOR. Here, we provide additional results for ROUGE-1-F, ROUGE-2-F, and Token Recovery Rate (TRR), as well as results over more training steps. Tab.~\ref{tab_appendix:big_llama3}, \ref{tab_appendix:big_llama2}, and \ref{tab_appendix:big_deepseek} report the results for Llama-3-8B, Llama-2-7B, and DeepSeek-LLM-7B, respectively. We further report the defense performance against the BiSR(b) paradigm when cosine distance is used as the matching objective, as shown in Tab.~\ref{tab_appendix:cos}.

\subsection{Robust Training under Extreme Scale Blinding}
While the default mean scale \(m=1750\) provides a favorable privacy-utility trade-off, we further investigate whether Gradient Mirage remains trainable under extremely large scaling factors. As \(m\) becomes excessively large, the amplified gradients introduced by Scale Blinding may cause optimization instability, particularly for the Trunk segment whose gradients are directly affected by the perturbed objective.

To address this issue, we adopt a simple Trunk-Frozen Training strategy: the gradients of the Trunk segment are discarded, while the Top segment continues to be optimized normally and the Bottom segment is updated using the recovered gradients produced by Bottom-Gradient Recovery. Since the recovered Bottom gradients remove the expected amplification introduced by Scale Blinding, they remain relatively stable and informative for optimization.

We evaluate this strategy under significantly larger mean scales \(m\in\{4\times10^3,10^4,2.5\times10^4,5\times10^4,10^5\}\). As shown in Fig.~\ref{fig_Appendix:scale_large}, the proposed training strategy maintains stable convergence across all tested scales and consistently outperforms Top-Only Training.

\begin{table*}[t]
\centering
\renewcommand{\arraystretch}{1.05}
\resizebox{0.95\textwidth}{!}{%

\begin{tabular}{cccccccccc}
\toprule
Dataset     & Method      & Final PPL     &  PPL@500steps & PPL@1000steps & RougeL-F           & Rouge1-F           & Rouge2-F           & Meteor             & TRR                \\ 
\midrule
\multirow{7}{*}{\begin{tabular}[c]{@{}c@{}}Code-\\ Alpaca\end{tabular}} & Grad Pruning   & 5.389±.050      & 5.305±.027                                              & 5.475±.094                                               & .991±.001          & .991±.001          & .968±.004          & .967±.001          & .976±.002          \\
                             & GradSeq-LDP(8192)  & 5.045±.049   & 5.319±.125  & 5.358±.099   & .598±.015  & .605±.015  & .376±.028  & .540±.026   & .738±.015       \\
                             & GradSeq-LDP(10240) & 5.000±.098      & 5.327±.100     & 5.252±.160   & .741±.020 & .746±.018      & .637±.024    & .709±.030          & .844±.014     \\
                             & Grad Dropout     & 6.328±.333      & 6.371±.575      & 7.167±1.046     & .681±.040   & .682±.040   & .584±.041    & .682±.041          & .718±.044      \\
                             & \textbf{Gradient Mirage(512)}         & 5.304±.223      & 5.397±.158     & 5.404±.028    & \textbf{.070±.023} & \textbf{.071±.024} & \textbf{.030±.013} & \textbf{.045±.015} & \textbf{.163±.030} \\
                             
                             & \textbf{Gradient Mirage(1024)}         & 5.149±.135      & 5.260±.160     & 5.292±.024    & \textbf{.105±.007} & \textbf{.107±.008} & \textbf{.028±.008} & \textbf{.062±.009} & \textbf{.204±.011} \\
                             & Top-Only Training  & 5.438±.121      & 5.801±.042            & 5.636±.045       & -          & -          & -          & -           & -       \\
                             & Standard Training  & 4.378±.001      & 4.495±.001     & 4.418±.001      & .998±.000      & .998±.000     & .984±.001   & .977±.001          & .983±.001          \\ 
\midrule
\multirow{7}{*}{PIQA}                                                   & Grad Pruning   & 14.387±2.094    & 13.592±.867     & 14.741±1.939   & .942±.016    & .944±.015   & .882±.026         & .928±.024   & .931±.019      \\
                                & GradSeq-LDP(8192)  & 6.710±.235      & 6.921±.032       & 6.798±.105    & .819±.020          & .822±.019       & .659±.039          & .763±.027          & .831±.005          \\
                                 & GradSeq-LDP(10240) & 6.717±.274      & 6.870±.022     & 6.749±.073                & .845±.010          & .848±.009     & .694±.008          & .712±.019          & .827±.014          \\
                                & Grad Dropout       & 301.776±415.971 & 11.223±2.122      & 118.565±154.226        & .844±.013          & .854±.024          & .773±.017          & .820±.025          & .848±.008          \\
                                & \textbf{Gradient Mirage(512)}    & 6.673±.164      & 6.917±.162                  & 6.923±.295     & \textbf{.071±.005} & \textbf{.072±.005} & \textbf{.008±.005} & \textbf{.036±.004} & \textbf{.135±.009} \\
                                
                                 & \textbf{Gradient Mirage(1024)}    & 6.451±.015      & 7.089±.081                   & 6.940±.195     & \textbf{.142±.009} & \textbf{.148±.010} & \textbf{.021±.008} & \textbf{.081±.014} & \textbf{.221±.012} \\
                                                                        & Top-Only Training  & 7.444±.130      & 7.935±.064                                              & 7.688±.008                                               & -                  & -                  & -                  & -                  & -                  \\
                                                                        & Standard Training  & 5.794±.008      & 6.022±.002                                              & 5.870±.016                                               & .999±.002          & .999±.002          & .979±.009          & .968±.012          & .966±.007          \\ 
\midrule
\multirow{7}{*}{GSM8K}                                                  & Grad Pruning  & 39.113±35.004   & 7.430±.183    & 7.766±.136    & .792±.013     & .796±.015          & .695±.024          & .753±.029          & .792±.024          \\
                                 & GradSeq-LDP(8192)  & 12.126±.119     & 13.091±.384                                             & 12.667±.296                                              & .312±.017          & .332±.015          & .131±.023          & .214±.027          & .475±.009          \\
                                     & GradSeq-LDP(10240) & 12.130±.115     & 12.981±.326                                             & 12.507±.313                                              & .344±.014          & .364±.018          & .156±.021          & .241±.025          & .513±.008          \\
                                     & Grad Dropout       & 16.104±3.461    & 17.533±3.572                                            & 17.125±3.352                                             & .843±.024          & .852±.025          & .784±.040          & .848±.030          & .909±.011          \\
                                        & \textbf{Gradient Mirage(512)}         & 11.989±.105     & 12.796±.149     & 12.327±.165   & \textbf{.028±.002} & \textbf{.030±.001} & \textbf{.000±.000} & \textbf{.012±.001} & \textbf{.078±.013} \\
                                        & \textbf{Gradient Mirage(1024)}         & 12.037±.294     & 13.006±.181             & 12.659±.100         & \textbf{.074±.005} & \textbf{.077±.006} & \textbf{.004±.000} & \textbf{.032±.003} & \textbf{.163±.012} \\
                                      & Top-Only Training  & 13.719±.105     & 14.407±.478                                             & 13.873±.390                                              & -                  & -                  & -                  & -                  & -                  \\
                                                                        & Standard Training  & 10.007±.002     & 10.173±.021                                             & 10.094±.008                                              & .978±.002          & .978±.002          & .962±.003          & .976±.001          & .986±.001          \\ 
           \bottomrule
\end{tabular}

}

\caption{\textit{Comparison between Gradient Mirage and existing privacy-preserving methods with DeepSeek-LLM-7B.}}
\label{tab_appendix:big_deepseek}

\end{table*}

\begin{table*}[t]
\centering
\resizebox{0.92\textwidth}{!}{%
\begin{tabular}{cc|ccc|ccc|ccc}
\toprule
\multicolumn{2}{c|}{Dataset}                                                                                     & \multicolumn{3}{c|}{CodeAlpaca}   & \multicolumn{3}{c|}{PIQA}         & \multicolumn{3}{c}{GSM8K}         \\ 
\midrule
\multicolumn{1}{c|}{Model}                                                                       & Method        & RougeL-F  & Rouge2-F  & Meteor    & RougeL-F  & Rouge2-F  & Meteor    & RougeL-F  & Rouge2-F  & Meteor    \\ 
\midrule
\multicolumn{1}{c|}{\multirow{7}{*}{\begin{tabular}[c]{@{}c@{}}Llama3\\ -8B\end{tabular}}}       & GP(0.7)       & .488±.006 & .181±.002 & .347±.002 & .691±.003 & .444±.007 & .601±.001 & .473±.008 & .179±.012 & .307±.011 \\
\multicolumn{1}{c|}{}                                                                            & GD(0.6, 5e-3) & .089±.016 & .009±.009 & .061±.010 & .237±.010 & .088±.012 & .186±.007 & .066±.012 & .000±.000 & .039±.003 \\
\multicolumn{1}{c|}{}                                                                            & LDP(10240)    & .103±.012 & .010±.005 & .071±.010 & .309±.007 & .106±.003 & .234±.008 & .073±.006 & .003±.002 & .046±.003 \\
\multicolumn{1}{c|}{}                                                                            & LDP(12288)    & .132±.005 & .018±.007 & .093±.004 & .322±.012 & .114±.005 & .247±.011 & .080±.002 & .003±.001 & .046±.003 \\
\multicolumn{1}{c|}{}                                                                            & Standard      & .557±.012 & .254±.015 & .416±.013 & .716±.007 & .448±.009 & .625±.009 & .469±.016 & .177±.016 & .317±.016 \\
\cmidrule{2-11} 
\multicolumn{1}{c|}{}                                                                            & Ours(512)     & .094±.005 & .002±.001 & .068±.003 & .121±.017 & .013±.002 & .082±.004 & .103±.004 & .003±.001 & .066±.005 \\
\multicolumn{1}{c|}{}                                                                            & Ours(1024)    & .161±.011 & .008±.002 & .106±.003 & .246±.019 & .040±.009 & .156±.012 & .136±.015 & .012±.002 & .105±.010 \\ 
\midrule
\multicolumn{1}{c|}{\multirow{7}{*}{\begin{tabular}[c]{@{}c@{}}Llama2\\ -7B\end{tabular}}}       & GP(0.7)       & .478±.017 & .285±.029 & .416±.027 & .629±.017 & .464±.021 & .573±.022 & .393±.015 & .202±.015 & .334±.016 \\
\multicolumn{1}{c|}{}                                                                            & GD(0.6, 5e-3) & .016±.007 & .001±.001 & .008±.003 & .037±.014 & .003±.002 & .021±.008 & .012±.001 & .000±.000 & .008±.001 \\
\multicolumn{1}{c|}{}                                                                            & LDP(10240)    & .092±.017 & .016±.003 & .054±.007 & .225±.010 & .069±.011 & .152±.016 & .026±.002 & .002±.000 & .015±.001 \\
\multicolumn{1}{c|}{}                                                                            & LDP(12288)    & .103±.011 & .019±.004 & .061±.007 & .235±.016 & .077±.014 & .159±.019 & .033±.002 & .003±.001 & .019±.003 \\
\multicolumn{1}{c|}{}                                                                            & Standard      & .471±.008 & .287±.012 & .414±.016 & .617±.022 & .444±.021 & .570±.012 & .420±.004 & .226±.005 & .371±.008 \\
\cmidrule{2-11} 
\multicolumn{1}{c|}{}                                                                            & Ours(512)     & .066±.009 & .006±.006 & .041±.006 & .085±.007 & .007±.005 & .061±.004 & .046±.006 & .001±.001 & .030±.003 \\
\multicolumn{1}{c|}{}                                                                            & Ours(1024)    & .110±.017 & .018±.007 & .066±.003 & .131±.005 & .014±.005 & .082±.011 & .072±.007 & .003±.002 & .040±.001 \\ 
\midrule
\multicolumn{1}{c|}{\multirow{7}{*}{\begin{tabular}[c]{@{}c@{}}DeepSeek\\ -LLM-7B\end{tabular}}} & GP(0.8)       & .381±.018 & .192±.014 & .289±.022 & .549±.021 & .347±.024 & .446±.023 & .361±.009 & .175±.007 & .249±.008 \\
\multicolumn{1}{c|}{}                                                                            & GD(0.6, 5e-3) & .002±.001 & .000±.000 & .002±.001 & .001±.001 & .000±.000 & .002±.001 & .003±.001 & .000±.000 & .002±.000 \\
\multicolumn{1}{c|}{}                                                                            & LDP(8192)     & .033±.002 & .003±.001 & .016±.001 & .180±.011 & .051±.011 & .109±.005 & .020±.003 & .001±.000 & .009±.002 \\
\multicolumn{1}{c|}{}                                                                            & LDP(10240)    & .032±.002 & .003±.001 & .015±.001 & .187±.010 & .057±.007 & .113±.005 & .025±.005 & .000±.000 & .011±.002 \\
\multicolumn{1}{c|}{}                                                                            & Standard      & .434±.017 & .250±.015 & .358±.013 & .644±.006 & .473±.023 & .576±.022 & .383±.022 & .211±.015 & .307±.016 \\

\cmidrule{2-11} 

\multicolumn{1}{c|}{}                                                                            & Ours(512)     & .015±.002 & .001±.001 & .007±.000 & .072±.016 & .005±.003 & .034±.007 & .018±.004 & .000±.000 & .007±.002 \\
\multicolumn{1}{c|}{}                                                                            & Ours(1024)    & .057±.008 & .005±.001 & .028±.003 & .166±.020 & .021±.007 & .078±.012 & .066±.005 & .002±.001 & .030±.004 \\ 
\bottomrule

\end{tabular}

}

\caption{\textit{Comparison of defense performance against the BiSR(b) attack paradigm using cosine distance as the matching objective across three models.}}
\label{tab_appendix:cos}

\end{table*}

\begin{figure*}[t]
    \centering
    \includegraphics[width=0.94\linewidth]{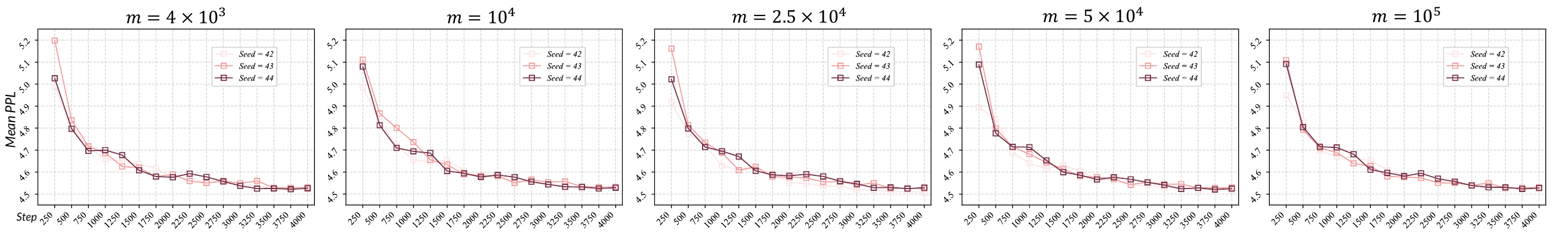}
    \caption{\textit{Fine-tuning performance under extremely large mean scales $m$ with the Trunk-Frozen Training strategy, evaluated on Llama-3-8B with the CodeAlpaca dataset.} The proposed strategy maintains stable convergence across all tested scales and random seeds.}
    \label{fig_Appendix:scale_large}
    \vspace{-2pt}
\end{figure*}

\clearpage

\begin{figure*}[t]
    \centering
    \includegraphics[width=0.96\linewidth]{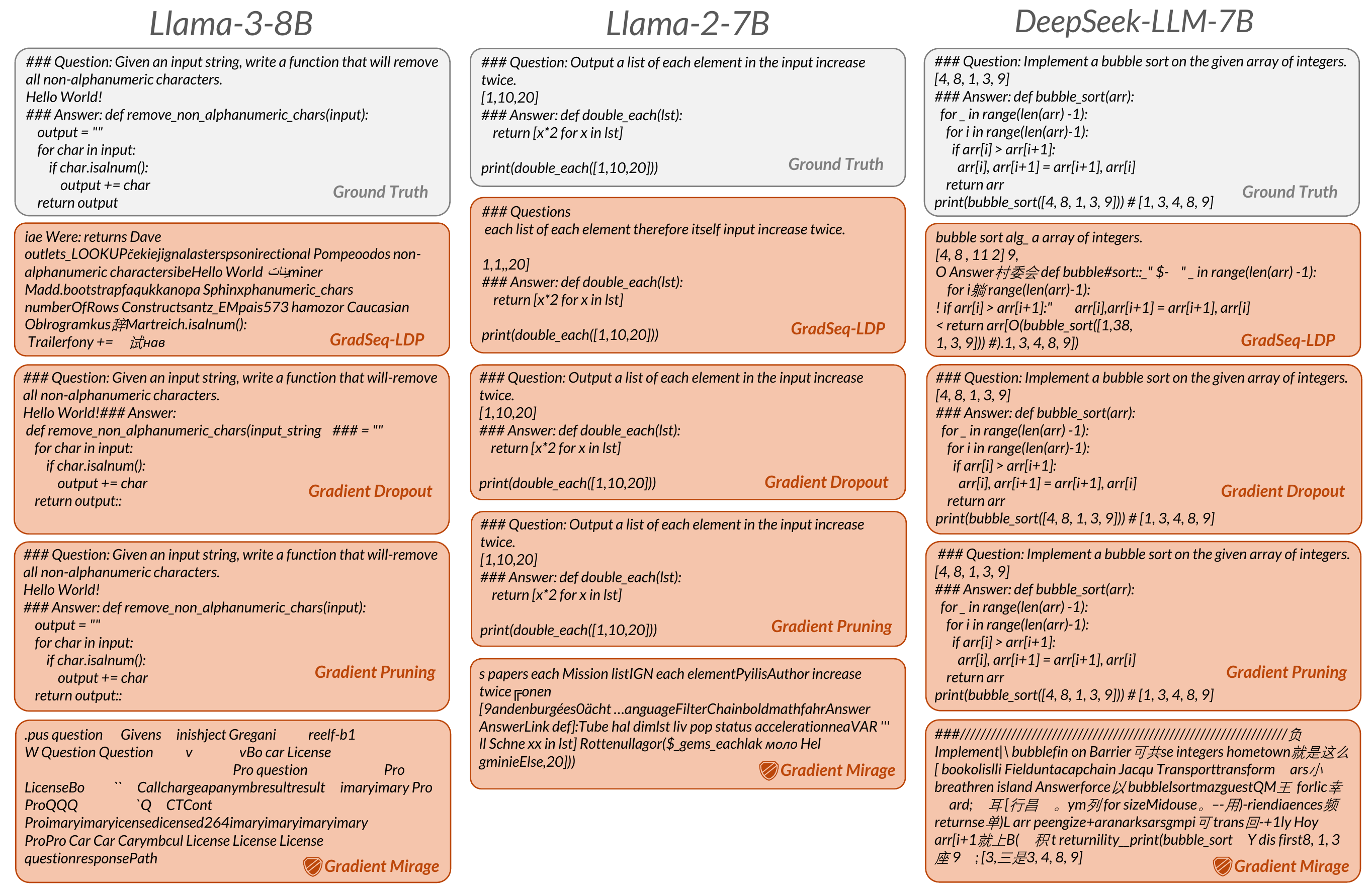}
    \caption{\textit{Qualitative comparison of reconstructed inputs under different privacy-preserving methods across Llama-3-8B, Llama-2-7B, and DeepSeek-LLM-7B.} Compared with existing defenses, Gradient Mirage yields substantially more distorted and incoherent reconstructions.}
    \label{fig_Appendix:vis}
\end{figure*}

\begin{figure*}[t]
    \centering
    \includegraphics[width=0.98\linewidth]{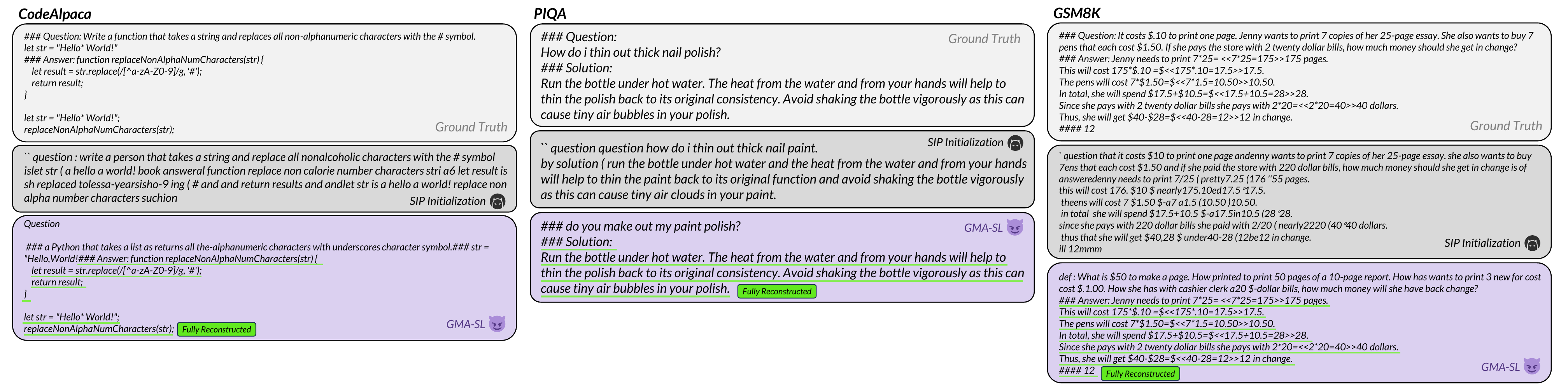}
    \caption{\textit{Qualitative results of GMA-SL under instruction fine-tuning across CodeAlpaca, PIQA, and GSM8K.} Although the query tokens are excluded from supervision during user-side training, GMA-SL successfully reconstructs the supervised answer portion.}
    \label{fig_Appendix:GMA-SFT}
\end{figure*}
\clearpage

\section{Visualization of Gradient Mirage Performance}
Fig.~\ref{fig_Appendix:vis} presents representative defense performance across three LLMs. Compared with existing defense methods, Gradient Mirage produces substantially more distorted and incoherent reconstructions, revealing considerably less information about the ground-truth input. These examples visually demonstrate the effectiveness of Gradient Mirage in defending against GMA-SL.

\section{Future Work}
Beyond the settings evaluated in our main experiments, we additionally find that GMA-SL remains effective in the instruction fine-tuning setting. Under all attack settings considered for GMA-SL, the labels used in user-side training are constructed by left-shifting the entire sequence by one token, meaning that supervision is applied to all tokens. However, we further find that under instruction fine-tuning, where the query part is not supervised, GMA-SL still poses a substantial threat to the answer part. 

GMA-SL matches the interface gradients computed under full-sequence supervision. However, we apply the same attack assumption to the instruction fine-tuning setting, where only the answer tokens are supervised. Let the sequence be decomposed as \(x = (x^{q}, x^{a})\), where \(x^{q}\) and \(x^{a}\) denote the query and answer parts, respectively. The optimization objective can be written as
$
\mathcal{L}_{\text{ans}} = \sum_{t \in \mathcal{I}_{a}} \ell\!\left(f_\theta(x)_{t}, y_t\right),
$
where \(\mathcal{I}_{a}\) is the set of positions belonging to the answer part. Even under this mismatch, namely, attacking the instruction fine-tuning scenario with the full-supervision assumption, we find that the reconstruction quality for the answer part remains remarkably high, whereas the reconstruction quality for the query part is relatively poor. Representative examples on CodeAlpaca, GSM8K, and PIQA are shown in Fig.~\ref{fig_Appendix:GMA-SFT}.

Looking forward, we plan to extend Gradient Mirage to the study of gradient privacy attacks and defenses in multimodal large language models. Several deeper questions also warrant further investigation, including the underlying mechanisms of GMA-SL and, in particular, why it remains effective under instruction fine-tuning despite supervision being restricted to the answer portion. Moreover, unlike supervised learning, reinforcement learning does not rely on explicit token-level labels. In RL-based SL, it remains an open question how to formally define the privacy leakage of reward models, as well as how an adversary could effectively extract private information from reward models.

\end{document}